\documentclass{opt2025}

\usepackage{amsmath}
\usepackage{amssymb}
\usepackage{algorithm}
\usepackage{wrapfig}
\usepackage{xspace}
\usepackage{graphicx}
\usepackage{wrapfig}
\usepackage{float}
\usepackage{url}
\usepackage{color}
\usepackage{multirow}
\usepackage{gensymb}
\usepackage{longtable}
\usepackage{tikz}
\usepackage{comment}
\usepackage{dsfont}
\usepackage{framed}
\usepackage{enumitem}
\usepackage{bm}
\usepackage{pifont}
\usepackage{algpseudocode}
\usepackage{booktabs}
\usepackage{multibib}
\usepackage{makecell}
\usepackage{textcomp}

\newcommand{\ouralg}{\texttt{D-DRO}\xspace}

\newcommand{\dagalg}{\texttt{InnerMax}\xspace}
\newcommand{\dro}{DRO\xspace}

\newcommand{\kldro}{\texttt{KL-DRO}\xspace}

\newcommand{\wdro}{\texttt{W-DRO}\xspace}
\newcommand{\ml}{\texttt{ML}\xspace}
\newcommand{\dml}{\texttt{DML}\xspace}
\newcommand{\ood}{OOD\xspace}
\newcommand{\kl}{KL\xspace}

\newcommand{\ddpm}{DDPM\xspace}
\newcommand{\zinit}{z_0\xspace}
\newcommand{\ztheta}{z_\theta\xspace}
\newcommand{\sdes}{SDEs\xspace}

\title[Distributionally Robust Optimization via Diffusion Ambiguity Modeling]{Distributionally Robust Optimization via Diffusion Ambiguity Modeling}

\optauthor{
\Name{Jiaqi Wen} \Email{jwen4@UH.EDU}\\
\Name{Jianyi Yang} \Email{jyang66@UH.EDU}\\
\addr University of Houston, USA}

\begin{document}

\maketitle

\begin{abstract}

This paper studies Distributionally Robust Optimization (DRO), a fundamental framework for enhancing the robustness and generalization of statistical learning and optimization. An effective ambiguity set for DRO must involve distributions that remain consistent with the nominal distribution while being diverse enough to account for a variety of potential scenarios. Moreover, it should lead to tractable DRO solutions. To this end, we propose a diffusion-based ambiguity set design that captures various adversarial distributions beyond the nominal support space while maintaining consistency with the nominal distribution. Building on this ambiguity modeling, we propose Diffusion-based DRO (\ouralg), a tractable DRO algorithm that solves the inner maximization over the parameterized diffusion model space. We formally establish the stationary convergence performance of \ouralg and empirically demonstrate its superior Out-of-Distribution (OOD) generalization performance in a ML prediction task.

\end{abstract}


\section{Introduction}
Distributionally Robust Optimization (DRO) is a fundamental framework for enhancing the robustness of statistical learning and optimization problems, particularly under Out-of-Distribution (\ood) scenarios \cite{OOD_generalization_liu2021towards,OOD_ML_arjovsky2020out}.
DRO formulates a minimax optimization problem, where the inner maximization identifies the worst-case distribution within an ambiguity set, and the outer minimization optimizes the decision variable against this worst-case scenario \cite{DRL_chen2020distributionally, DRQL_liu2022distributionally, DR_regression_chen2018robust, WDRL_kuhn2019wasserstein, DRRL_smirnova2019distributionally}.
Unlike non-probabilistic robust optimization, DRO leverages probabilistic uncertainty modeling to enable improved generalization performance. This property has made DRO increasingly important in ML for addressing distribution shifts, noisy data, and adversarial conditions \cite{distribution_shift_quinonero2022dataset,robut_learning_simulation_muratore2022robot,DRL_environment_generation_ren2022distributionally,noisy_label_li2019learning,noisy_data_veit2017learning}.

The performance of \dro algorithms critically depends on the design of the ambiguity set, which must contain meaningful distributional variations around the nominal distribution. A common approach is to model the ambiguity set using $\phi$-divergences, such as the Kullback--Leibler (KL) divergence~\cite{KL-DRO_hu2013kullback, DR_BO_divergence_husain2023distributionally, DRBO_kirschner2020distributionally, Conic_reformulation_KL_DRO_kocuk2020conic, DRQL_liu2022distributionally}. 
Although such $\phi$-divergence-based formulations can sometimes yield closed-form solutions~\cite{KL-DRO_hu2013kullback,DR_BO_divergence_husain2023distributionally}, they require that any distribution $P$ in the ambiguity set be \textit{absolutely continuous} with respect to the nominal distribution $P_0$ (denoted $P \ll P_0$), meaning that for any measurable set $\mathcal{A}$, if $P_0(\mathcal{A}) = 0$, then $P(\mathcal{A}) = 0$. This implicit constraint limits robustness in scenarios with support shifts. 
In contrast, Wasserstein-DRO leverages the Wasserstein distance to define the ambiguity set, allowing for support shifts. However, solving Wasserstein-DRO over the infinite probability space is difficult. Some approaches~\cite{DRO_mohajerin2018data,DRL_chen2020distributionally,Wasserstein_DRO_gao2023distributionally} reformulate Wasserstein-DRO as a finite-dimensional optimization problem based on the convex assumptions, which typically do not hold in ML. Other methods approximate Wasserstein-DRO via adversarial optimization~\cite{staib2017distributionally,Wasserstein_DRO_gao2023distributionally}, but such relaxations are overly conservative, limiting their ability to fully leverage the benefits of probabilistic uncertainty modeling.

Recent advances have aimed to address the challenges of ambiguity modelling in \dro. 
For example, \cite{NEURIPS2022_da535999} incorporates data geometric properties into the design of discrepancy metrics, thereby reducing the complexity of the ambiguity set. 
In addition, a recent work~\cite{sinkhorn_DRO_yang2025nested} studies \dro with ambiguity sets defined by a generalized Sinkhorn distance, which enables modelling uncertainty across distributions with different support spaces. 
Another work~\cite{DRL_environment_generation_ren2022distributionally} constructs a Wasserstein-based ambiguity set in the latent space of generative models and subsequently applies Wasserstein-\dro methods to solve the problem. 
Additional related studies are discussed in Appendix~\ref{sec:related_work}.

Different from these approaches, our work is the first to model the ambiguity set in the space of diffusion models, which offers several advantages:  
(1) Diffusion models have a strong capability to represent the underlying data distribution, ensuring that the distributions in the ambiguity set remain consistent with the nominal distribution.  
(2) Diffusion models are capable of producing diverse samples beyond the training support space, thereby enabling the discovery of worst-case distributions.
(3) Diffusion models provide a finite, parameterized optimization space, avoiding the need to solve problems over an infinite probability space.

Our main contributions are summarized below:  
\textbf{(1)} 
We introduce a novel \emph{Diffusion Ambiguity Set} for \dro, which encompasses diverse distributions while preserving consistency with the nominal distribution.
\textbf{(2)} We design an inner maximization procedure for \dro with the proposed Diffusion Ambiguity Set, enabling tractable iterative optimization within a finite, parameterized space.  
\textbf{(3)} We propose \ouralg (Algorithm~\ref{DAG-DRL}), which solves the resulting minimax optimization problem, and formally establish its stationary convergence in Theorems~\ref{thm:analysis} and~\ref{thm:convergence_minmax}.  
\textbf{(4)} We demonstrate the superior performance of \ouralg on the challenging ML task of renewable energy prediction.

\section{DRO and Ambiguity Modeling}
\textbf{Distributionally Robust Optimization.}
 \dro optimizes for the worst-case performance given an ambiguity set constructed on a nominal distribution $P_0$, which can be an empirical distribution $S_0$.  Consider an objective function $f(w,x)$ with the decision variable $w\in\mathcal{W}$ and the random parameter $x\in \mathcal{X}$. Given the nominal distribution $P_0(x)$ of the random parameter $x$, DRO solves the following minimax optimization problem.
\begin{equation}\label{eqn:DRO}
w = \min_{w\in \mathcal{W}}\max_{P\in \mathcal{B}(P_0,\epsilon)} \mathbb{E}_{x\sim P}[f(w,x)],
\end{equation}
where $\mathcal{B}(P_0,\epsilon)$ is the ambiguity set containing possible testing distributions which is typically modeled as a distribution ball $\mathcal{B} (P_0,\epsilon)= \left\{ P \mid D(P,P_0) < \epsilon \right\}$ given a distribution discrepancy measure $D$ and an adversary budget $\epsilon$. 

\textbf{Ambiguity modeling in \dro.}  
The choice of the ambiguity set in \dro has a significant impact on both generalization performance and solution tractability. We observe that a well-designed ambiguity set in \dro should satisfy the following properties:  
First, it should include diverse distributions beyond the support of the nominal distribution, enabling the identification of various worst-case distributions.  
Second, the distributions within the ambiguity set should remain realistic and consistent with the nominal distribution, balancing the average-case and worst-case performance.  
Finally, the ambiguity set should facilitate a tractable solution of the \dro problem despite the infinite probability space.

\section{Method}
\subsection{Diffusion Ambiguity Modeling}
\textbf{Diffusion models}. Diffusion models learn an underlying distribution $P_0$ from its finite dataset $S_0$ and can generate diverse samples from this distribution. The model training and inference are based on a forward process and a reverse process, detailed in Appendix~\ref{sec:dm}.  
The forward process begins with an initial sample $x_0 \in S_0$ and evolves according to a stochastic process to produce random variables $x_1, \dots, x_T$ with marginal distributions $P_t$, $t \in [T]$. The reverse process starts with $x_T$ drawn from a prior distribution $\pi$ that approximates $P_T$ and reconstructs $x_{T-1}, \dots, x_0$ by following the reverse diffusion process, which depends on the score function (the gradient of the log-density of the underlying distribution $\nabla_x \log P_t(x)$).

The score-matching method employs an ML model $s_\theta(x,t)$ with parameters $\theta$ to approximate the distribution gradient $\nabla_x \log P_t(x)$ by minimizing an empirical score-matching loss $J(\theta, S_0)$ on the dataset $S_0$.  
Once the score model $s_\theta(x,t)$ is obtained, we can generate samples $x_0, \dots, x_{T-1}$ according to the distribution
\(
P_\theta(x_{0:T}) = \pi(x_T) \prod_{t=1}^T P_\theta(x_{t-1} \mid x_t), 
\) where  \(P_\theta(x_{t-1} \mid x_t) = \mathcal{N}\big(x_{t-1}; \mu_\theta(x_t, t), \Sigma_\theta(x_t, t)\big)\) where $\mu_\theta(x_t, t)$ and $\Sigma_\theta(x_t, t)\big)$ rely on the score model parameter $\theta$.

Theorem 1 in \cite{MLE_score_based_song2021maximum} (restated in Lemma~\ref{KL_bound_diffusion}) shows that if the score-matching loss is bounded, i.e., $J(\theta, S_0) \le \epsilon$, then the KL divergence $D_{\mathrm{KL}}(P_0 \,\|\, P_\theta)$ is bounded by $\epsilon$ plus additional approximation error terms.  
We note that the KL divergence $D_{\mathrm{KL}}(P_0 \,\|\, P_\theta)$ in Lemma~\ref{KL_bound_diffusion} differs from the KL divergence $D_{\mathrm{KL}}(P_\theta \,\|\, P_0)$ commonly used in \kldro. Importantly, the former allows $P_\theta$ to have a broader support than $P_0$ (i.e., $P_0 \ll P_\theta$).

\textbf{Diffusion ambiguity set}.
We can model the ambiguity set based on the score-matching loss of a diffusion model, due to its property to constrain the distributional discrepancy while allowing a broader support space. This leads to \dro with diffusion ambiguity sets:
\begin{equation}\label{eqn:objective}
\min_{w \in \mathcal{W}} \max_{\theta \in \Theta} \mathbb{E}_{x \sim P_\theta}[f(w, x)], 
\quad \mathrm{s.t.} \quad J(\theta, S_0) \le \epsilon,
\end{equation}
where $P_\theta$ denotes the  distribution of the diffusion reverse process, and $\epsilon$ is the adversarial budget.  

The diffusion ambiguity set enhances \dro performance as follows.
Due to the distribution modelling capability of diffusion models, the score-matching constraint ensures that the diffusion-modelled distributions remain consistent with the nominal data, mitigating over-conservativeness issues in \dro.  
It also leverages diffusion models’ ability to generate diverse samples beyond the nominal support, enabling the identification of worst-case distributions given any $w$. 
Furthermore, the inner maximization in \eqref{eqn:objective} operates in a finite, parameterized space, ensuring tractable optimization.

\subsection{Diffusion-Based Inner Maximization}

To solve the inner maximization of \eqref{eqn:objective}, which is a constrained optimization over the diffusion parameter space, we adopt a dual learning approach: introducing a Lagrangian dual $\mu>0$ to reformulate it as the unconstrained problem
\begin{equation}\label{eqn:lagrangian_objective}
\max_{\theta} \; \mathbb{E}_{x\sim P_{\theta}}[f(w,x)] - \mu J(\theta, S_0),
\end{equation}
and updating $\mu$ via dual gradient descent. As is shown by \dagalg in Algorithm \ref{DAG-DRL}, we increase $\mu$ when $J$ exceeds the budget $\epsilon$ and decrease it otherwise. 

 We apply the policy optimization methods to transform the objective in \eqref{eqn:lagrangian_objective} into a tractable form. Vanilla policy gradient \cite{policy_gradient_sutton1999policy} directly calculates the empirical gradient of the objective in \eqref{eqn:lagrangian_objective} as:
 \begin{equation}\label{eqn:vpg}
  \max_{\theta} \hat{\mathbb{E}}_{ P_{\theta}(x_{0:T})}[
    \ln P_{\theta}(x_{0:T}) \cdot f(w,x_0)]-\mu\cdot J(\theta, S_0),
\end{equation}
where $\hat{\mathbb{E}}_{ P_{\theta}(x_{0:T})}$ is the empirical mean based on the $T-$step samples for the backward process of the diffusion model $P_{\theta}$, and $\ln P_{\theta}(x_{0:T})=-\sum_{t=1}^T[x_{t-1}-\mu_{\theta}(x_t,t)]^2+C_2$ where $C_2$ is a constant. The derivation details are given in Appendix~\ref{sec:policy_gradient}.

 Proximal Policy Optimization (PPO) \cite{PPO_schulman2017proximal} is believed to have more stable performance. It transforms the objective  \eqref{eqn:lagrangian_objective} into a differentiable form as follow:
 \begin{equation}\label{eqn:ppo}
  \max_{\theta} \hat{\mathbb{E}}_{P_{\theta_{\mathrm{old}}}(x_{0:T})}[
   \min(r_{\theta}(x_{0:T})f(w,x_0)), \mathrm{clip}(r_{\theta}(x_{0:T}),1-\kappa,1+\kappa)\cdot f(w,x_0))]-\mu\cdot J(\theta, S_0),
\end{equation}
where  $P_{\theta_{\mathrm{old}}}$ is a reference diffusion model used for sampling the backward sequence, the probability ratio is $r_{\theta}(x_{0:T})=\frac{P_{\theta }(x_{0:T})}{P_{\theta_{\mathrm{old}}}(x_{0:T})}=\exp\{-\sum_{t=1}^T(\frac{\|x_{t-1}-\mu_{\theta}(x_t,t)\|^2}{2\sigma_{t}^2}-\frac{\|x_{t-1}-\mu_{\theta_{\mathrm{old}}}(x_t,t)\|^2}{2\sigma_{t}^2})\}$, and $\kappa\in(0,1)$ is the clipping parameter to avoid overly-large policy updates. To reduce the training complexity, instead of optimizing all the $T$ steps, we can only optimize the last $T'$ steps of the backward process by choosing $r_{\theta}(x_{0:T'})=\exp\{-\sum_{t=1}^{T'}(\frac{\|x_{t-1}-\mu_{\theta}(x_t,t)\|^2}{2\sigma_{t}^2}-\frac{\|x_{t-1}-\mu_{\theta_{\mathrm{old}}}(x_t,t)\|^2}{2\sigma_{t}^2})\}$ and only keep the loss terms for corresponding steps in $J(\theta, S_0)$. This practice has been verified by our experiments in Appendix~\ref{app:exp}. More derivation details are given in Appendix~\ref{sec:ppo}

\begin{algorithm}[t!]
    \caption{Diffusion-based \dro (\ouralg)}
    \label{DAG-DRL}
    \KwIn{Training dataset $S_0$; Adversary budget $\epsilon>0$; Step size $\eta>0$, $\lambda>0$.}
    
    \textbf{Initialization:} Initialize decision variable $w$, diffusion parameter $\theta$ and Lagrangian weight $\mu>0$\;
    
    \For{$j=1,2,\cdots,I$}{
    \tcp{\scriptsize Diffusion-based inner maximization (\dagalg)}
        \For{$k=1,2,\cdots,K$}{
           \;\:Update diffusion parameter $\theta_k$ by solving \eqref{eqn:lagrangian_objective} given $\mu$\; \\
            Update Lagrangian parameter: $\mu \leftarrow \max\{0,\mu + \eta \big(J(\theta_k,S_0)-\epsilon\big)\}$\;
        }
        \tcp{\scriptsize Outer minimization to update decision variable}
         \;Generate dataset $S_{j}$ with diffusion model $P_{\theta^{(j)}}$ with $\theta^{(j)}$ uniformly selected from $[\theta_1,\cdots, \theta_K]$\;\\
        Update decision variable:$w$: $w_j=w_{j-1}-\lambda\cdot \nabla_w \mathbb{E}_{x\in S_{j}}[f(w_{j-1},x)]$\;
    }    \Return $w$ uniformly selected from $[w_1,\cdots, w_I]$\;
    
    \end{algorithm}

\subsection{\ouralg Algorithm}
We design \ouralg in Algorithm \ref{DAG-DRL} which solves the mini-max optimization following the framework of Gradient Descent with Max-Oracle (GDMO) in \cite{nonconvex_nonconcave_mini_max_jin2020local}. In each iteration, we first run \dagalg to search for the worst-case diffusion model $P_{\theta^{(j)}}$ that maximizes the expected loss of the current variable $w$. Next, we generate an adversarial dataset $S_{j}$ based on $P_{\theta^{(j)}}$ and use it to update $w$. The convergence of Algorithm~\ref{DAG-DRL} in proved in Theorem \ref{thm:analysis} and Theorem \ref{thm:convergence_minmax}.

\section{Analysis}
The problem in \eqref{eqn:DRO} is a nonconvex--nonconcave min--max optimization problem, which makes it challenging to obtain a solution with convergence guarantee. 
This section formally establishes an insightful convergence bound for \ouralg.
    \begin{theorem}[Convergence of Inner Maximization]\label{thm:analysis}
       Let $\theta^*$ be the optimal diffusion parameter that solves the inner maximization \eqref{eqn:objective} given a variable $w$. 
       If the expected score-matching loss is bounded as $J(\theta)\leq \bar{J}$ and the step size is chosen as $\eta\sim\mathcal{O}(\frac{1}{\sqrt{K}})$, the inner maximization error holds that    \begin{equation}\label{eqn:adversarial_stren}
         \Delta':= \mathbb{E}_{x\sim P_{\theta^*}}[f(w,x)]-\mathbb{E}_k \mathbb{E}_{x\sim P_{\theta_k}}[f(w,x)]\leq \frac{1}{\sqrt{K}} \max\{\epsilon, \bar{J}\}\|\mu^{(1)}\|,
       \end{equation}
       where the outer expectation is taken over the randomness of output selection.
       In addition, given the assumptions in \ref{sec:assumptions}, the \kl-divergence with respect to the nominal distribution is bounded as \begin{equation}\label{eqn:consistency}
    \mathbb{E}_k[D_{\mathrm{KL}}(P_0||P_{\theta_k})]\leq \epsilon + \frac{\max\{\epsilon, \bar{J}\}|\mu_C-\mu^{(1)}|}{\sqrt{K} (\mu_C-\mu^*)}+ D_{\mathrm{KL}}(P_T||\pi)+C_1 ,
 \end{equation}
    where $\mu_C>\mu^*$ with $\mu^*$ being the optimal dual variable, $C_1$ is a constant, $P_T$ is the output distribution of the forward process, and $\pi$ is the initial distribution of the reverse process. 
    \end{theorem}

Proofs of Theorem \ref{thm:analysis} are provided in Appendix \ref{sec:proofconvergencemax}. The bound shows that when the inner iteration number $K$ is sufficiently large, the inner maximization error will be small enough so that \ouralg can find a worst-case distribution in the Diffusion Ambiguity Set. 
Moreover, the reverse KL-divergence w.r.t. the nominal distribution is bounded by the budget $\epsilon$, the KL-divergence between $P_T$ and $\pi$, and the constant gap $C_1$ due to the approximation of the score matching loss \cite{MLE_score_based_song2021maximum}.This implies that we can adjust the adversarial budget $\epsilon$ to obtain a worst-case distribution that is consistent with the nominal distribution. 
    
\begin{theorem}[Convergence of \ouralg]
  \label{thm:convergence_minmax} Assume that the objective $f(w,x)$ is $\beta-$smooth and $L-$ Lipschitz with respect to $w$ and is upper bounded by $\bar{f}$. If each  dataset $S_j$ sampled from the diffusion model contains $n$ examples and the step size is chosen as $\lambda\sim\mathcal{O}(\sqrt{\frac{1}{\beta L^2 H}})$, then with probability $1-\delta, \delta\in(0,1)$, the average norm of the Moreau envelope of  $\phi(w):= \max_{\theta}\mathbb{E}_{P_{\theta}}[f(w,x)]$ satisfies
   \begin{equation}
\begin{split}
\mathbb{E}_{j,k}\left[\|\nabla\phi_{\frac{1}{2\beta}}(w)\|^2\right]
\leq 4\beta \Delta'+\frac{V_1}{\sqrt{n}}+\frac{V_2}{\sqrt{H}},
\end{split}
\end{equation}
where the expectation is taken over the randomness of output selection, $\Delta'$ is the error of inner maximization bounded in Theorem \ref{thm:analysis}, $V_1=8\beta \bar{f}\sqrt{\log(2/\delta)}$ and $V_2=4L\sqrt{(\phi_{\frac{1}{2\beta}}(w_{1})-\min_w\phi(w))\beta}$.
\end{theorem}

The proof of Theorem \ref{thm:convergence_minmax} is deferred to Appendix \ref{sec:proofconvergenceminmax}.
\ouralg has a good global convergence performance if its output $w$ is close enough to the globally optimal variable  $w^*=\arg\min_{w} \phi(w)$ where $\phi(w)=\max_{\theta}\mathbb{E}_{P_\theta}[f(w,x)]$ is the  maximization oracle.  
However, the maximization oracle $\phi(w)$ can be non-convex and non-smooth given a general diffusion model, so it is usually difficult get a provable global convergence guarantee. Despite the difficulty of this min-max optimization, a local convergence can provide a theoretical justification for the optimality of \ouralg. An algorithm is locally optimal if the gradient norm of the objective $\|\nabla\phi(w)\|$ converges to zero. For neural networks, a local minima is always close to a global minima based on empirical studies and theoretical analysis \cite{glocal_optimalty_DNN_yun2017global}. 
As detailed in Appendix \ref{sec:Moreau_envelop}, if the gradient norm of the Moreau envelope $\phi_{\rho}(w):=\min_{w'}\phi(w')+\frac{1}{2\rho}\|w-w'\|^2$ given a non-zero constant $\rho$ approaches zero, \ouralg approximately converges to a local minima \cite{nonconvex_nonconcave_mini_max_jin2020local}.  Intuitively, 
$\phi_{\frac{1}{2\beta}}(w)$ is a smooth surrogate of $\phi(w)$ by choosing $\rho=\frac{1}{2\beta}$ with $\beta$ being the smoothness parameter of $f$. If the gradient norm $\|\nabla\phi_{\frac{1}{2\beta}}(w)\|$ is small enough, the obtained $w$ is close enough to a local minima and a zero gradient norm of the Moreau envelope implies that $w$ is the local minima.

Theorem~\ref{thm:convergence_minmax} shows that with a sufficiently large iteration number $H$ and sampling size $n$, the average gradient norm of the Moreau envelope for the optimal inner maximization function $\phi(w)$ is bounded by the error $\Delta'$ of the maximization oracle.
Since $\Delta'$ decreases with the inner iteration number $K$, as proved by Theorem \ref{thm:analysis}, the average gradient norm of the Moreau envelope approaches zero given enough iterations and samples. Thus, \ouralg at least converges to a locally optimal solution of \eqref{eqn:DRO}.

\section{Experiment}
\begin{wraptable}{r}{0.5\textwidth} 
\vspace{-3mm}
\centering
\caption{Test MSE on different datasets(Partial).}
\renewcommand{\arraystretch}{1.2}
\resizebox{0.5\textwidth}{!}{
\begin{tabular}{lccccc}
\toprule
\multicolumn{1}{c}{\multirow{2}{*}{\makecell{\textbf{Datasets} \\ (\textbf{Wasserstein Distance})}}}   & \multicolumn{5}{c}{\textbf{Algorithms}} \\
\cmidrule(lr){2-6}
 & \textbf{\ouralg} & \textbf{\kldro} & \textbf{\wdro} &  \textbf{\dml}  & \textbf{\ml}\\
\midrule

\multicolumn{1}{c}{\makecell{\textbf{BANC\_22}  (0.0240)}} & \textbf{0.0047} &	0.0086 &	0.0073 &	0.0078 &	0.0183   \\
\multicolumn{1}{c}{\makecell{\textbf{BANC\_21}  (0.1213)}} & \textbf{0.0054} &	0.0112 &	0.0121 &	0.0093 &	0.0238   \\
\midrule

\multicolumn{1}{c}{\makecell{\textbf{QLD\_22}  (0.2782)}} & \textbf{0.0192} &	0.0379 &	0.0557 &	0.0352 &	0.0667   \\
\multicolumn{1}{c}{\makecell{\textbf{QLD\_21}  (0.3054)}} & \textbf{0.0186} &	0.0377 &	0.0574 &	0.0339 &	0.0696   \\
\midrule

\multicolumn{1}{c}{\makecell{\textbf{GB\_22}  (0.1255)}} & \textbf{0.0105} &	0.0197 &	0.0245 &	0.0172 &	0.0360    \\
\multicolumn{1}{c}{\makecell{\textbf{GB\_21}  (0.1359)}} & \textbf{0.0094} &	0.0181 &	0.0229 &	0.0158 &	0.0340    \\
\midrule

\multicolumn{1}{c}{\textbf{Average}} & \textbf{0.0163} &	0.0288 &	0.0342 &	0.0271 &	0.0450    \\
\midrule
\multicolumn{1}{c}{\textbf{Maximum}} & \textbf{0.0509} &	0.0831 &	0.0879 &	0.0834 &	0.0946   \\

\bottomrule
\end{tabular}
}
\label{main_table_part}
\end{wraptable}
In this section, we present numerical studies on an ML for a renewable prediction task based on the Electricity Maps\cite{electricitymaps_datasets_portal} datasets (experiment setups in Appendix \ref{app:exp}). A part of the results is given in Table~\ref{main_table_part} where \ouralg are compared with baselines in Appendix \ref{sec:baselines} on various OOD testing datasets in Appendix \ref{sec:datasets}. 
We observe that all methods perform better on datasets with smaller Wasserstein distribution shifts. While \dml and other DRO baselines outperform ML, \ouralg consistently outperforms them across all OOD datasets, owing to its diffusion-based ambiguity set that effectively captures worst-case yet realistic distributions.More evaluation results are provided in Appendix \ref{app:exp}.

\section{Conclusion}
In this paper, we propose \ouralg, which introduces a novel diffusion-based ambiguity modeling for \dro and develops \ouralg to solve the \dro with the diffusion ambiguity set. We prove the stationary convergence performance of \ouralg. The experiments demonstrate robust \ood generalization performance of \ouralg. 
Overall, this new \dro solution has the potential to enhance the robustness of critical statistical optimization and ML tasks under distribution shifts and imperfect data.

\section*{Acknowledgement}
This work was supported in part by University of
Houston Start-up Funds 74825 [R0512039] and 74833 [R0512042].

\newpage
\appendix

\section{Related Work}\label{sec:related_work}
\subsection{Distributionally Robust Optimization}
\dro algorithms are widely studied to improve the \ood generalization performance for various optimization and \ml tasks \cite{DRQL_liu2022distributionally,DR_BO_divergence_husain2023distributionally,DRL_environment_generation_ren2022distributionally}. 
$\phi$-divergence-based \dro \cite{DR_BO_divergence_husain2023distributionally,DRBO_kirschner2020distributionally,KL-DRO_hu2013kullback,Conic_reformulation_KL_DRO_kocuk2020conic} is one of the commonly-used DRO method.  A closed-form solution to the inner maximization of \eqref{eqn:DRO} with KL-divergence-based ambiguity set is provided by \cite{KL-DRO_hu2013kullback}.  While we can usually get tractable DRO solutions based on $\phi$-divergence, the definition of $\phi$-divergence requires any distribution $P$ in the ambiguity set to be absolutely continuous with respect to the nominal distribution, which limits its application in statistical learning tasks.  
Alternatively, many studies adopt Wasserstein distance-based \dro \cite{DR_BO_divergence_husain2023distributionally,WDRL_kuhn2019wasserstein,DRO_mohajerin2018data,DRL_chen2020distributionally,Wasserstein_DRO_gao2023distributionally}. Wasserstein measure has no restrictions on the support of the distribution, but it is difficult to get a tractable solution for \wdro. Some methods \cite{DR_BO_divergence_husain2023distributionally,WDRL_kuhn2019wasserstein,DRO_mohajerin2018data,DRL_chen2020distributionally,Wasserstein_DRO_gao2023distributionally} reformulate Wasserstein-constrained DRO into a tractable finite optimization based on the assumption of convex objectives, which typically do not hold in deep learning. Other methods relax Wasserstein-constrained DRO into an adversarial optimization problem \cite{DRO_adversarial_training_sinha2017certifying,staib2017distributionally,Wasserstein_DRO_gao2023distributionally}, but this relaxation can be overly conservative and cannot fully exploit the benefits of probabilistic ambiguity modelling.

A line of recent studies focuses on addressing the challenges in ambiguity modelling for DRO~\cite{flow_based_DRO_xu2024flow,DRO_iterative_zhu2024distributionally,DRL_environment_generation_ren2022distributionally,NEURIPS2022_da535999,ma2024differentiable,sinkhorn_DRO_yang2025nested}. 
Among them,
\cite{NEURIPS2022_da535999} incorporates data geometric properties into the design of discrepancy metrics, reducing the size of the ambiguity set. \cite{DRL_environment_generation_ren2022distributionally} constructs a novel ambiguity set on the latent space of generative models such that the adversarial distribution is realistic and applies Wasserstein-based DRO solutions. Ma \emph{et al.}~\cite{ma2024differentiable} propose a differentiable parameterized Second-Order Cone (SOC) to characterize the ambiguity set and develop an end-to-end framework in which an ML model is trained to predict the ambiguity set for downstream \dro tasks.  Moreover, a latest work \cite{sinkhorn_DRO_yang2025nested} introduces a regularized nonconvex \dro method with generalized Sinkhorn distance, reformulating the problem as a contextual nested stochastic optimization and proving convergence without assuming strong convexity or large batches.
 Different from these methods, we utilize the strong distribution learning capability of diffusion models to build the ambiguity set, enabling the discovery of worst-case and realistic distributions. At the same time, the proposed algorithm \ouralg converts \dro into a finite tractable problem in the diffusion parameter space.

\subsection{Generative Models for Robust Learning}
Generative models have been widely studied to generate adversarial samples for robust training
\cite{diffusion_adversarial_example_generation_dai2024advdiff,diffusion_adversarial_dai2024diffusion,generative_attack_feature_learning_xie2024generative,ADL_allocation_du2022adversarial}. The target of these works is to generate adversarial attacking examples, which is fundamentally different from the worst-case distribution generation, which is studied in this paper and aims to improve \ood generalization. A recent paper proposed DRAGEN, that \cite{DRL_environment_generation_ren2022distributionally} models the adversarial distribution on the latent space of a generative model. However, it still lies in the Wasserstein-based framework. 
Wang \emph{et al.}~\cite{Gen-DFL_wang2025gen} introduced a Generate-then-Optimize framework, where a diffusion model is trained to generate data for downstream statistical optimization with a focus on the conditional value-at-risk (CVaR) objective. Although related to our work, their method mainly targets risk mitigation within in-distribution settings, while \ouralg is specifically designed to improve robustness under \ood scenarios.

\section{Preliminaries of Diffusion Models}\label{sec:dm}
This paper exploits diffusion models to improve the performance of \dro, so we summarize the preliminaries about diffusion models in this section.
 We introduce a score-based diffusion modelling by Stochastic Differential Equations (\sdes) \cite{MLE_score_based_song2021maximum}.
They rely on forward and backward stochastic processes introduced as follows.

\textbf{Forward Process.} 
The forward process incrementally injects noise into the data, generating a sequence of perturbed samples. It begins with an initial sample $x_0\in\mathcal{R}^d$ drawn from the underlining distribution $P_0$, and evolves according to a stochastic process as:
\begin{equation}\label{eqn:forward}
\mathrm{d}x = b(x,t) \mathrm{d}t + r(t) \mathrm{d}w,
\end{equation}
where $b(\cdot,t): \mathcal{R}^d\rightarrow \mathcal{R}^d$ is a vector-valued function, $r(t)\in\mathcal{R}$, $w$ is a standard Wiener process and $\mathrm{d}w$ is white Gaussian noise. By the forward process, we get a collection of random variables \(\{x_t\}_{t\in [0,T]}\). We use $P_t$ to represent the distribution of $x_t$ and $P_{t\mid 0}$ to denote the conditional distribution of $x_t$ given $x_0\sim P_0$. 
With a sufficiently long time \(T\), the marginal distribution $P_T(x_T)$ approximates a tractable prior distribution $\pi(x)$ which is typically chosen as a standard Gaussian distribution.

\textbf{Reverse Process.} 

A reverse diffusion process is associated with the forward equation in \eqref{eqn:forward} and is expressed as 
\begin{equation}\label{eqn:backward}
\mathrm{d}x = \left( b(x,t) -r(t)^2 \nabla_x \log P_t(x) \right) \mathrm{d}t + r(t)\mathrm{d}\bar{w},
\end{equation}
where $\bar{w}$ is a standard Wiener process in the reverse-time direction, $\nabla_x \log P_t(x)$ is the time-dependent score function.

\textbf{Score Matching.}
In the reverse process, the score function \(\nabla_x \log P_t(x)\) plays a critical role in directing the dynamics. To estimate the score function \(\nabla_x \log P_t(x)\), we train a score-based model \(s_\theta(x, t)\) based on samples generated from the forward diffusion process. The score-based model should minimize the following score-matching loss:
\[
J_{\mathrm{SM}}(\theta) =\int_0^T \mathbb{E}_{P_t(x)} \left[\iota(t)\left\| \nabla_{x} \log P_{t}(x) - s_{\theta}(x, t) \right\|^2 \right] \mathrm{d}t,
\]
where $\iota(t)>0$ is a positive weighting function. We usually approximate the score-matching loss by a tractable denoising score-matching loss up to a constant that does not rely on $\theta$:
\begin{equation}\label{eqn:score_match}
J\!(\theta) \!\!=\!\!\!\int_0^T \!\!\!\!\!\mathbb{E}_{P_0(x)P_{t\mid 0}(x'\mid x)} \!\!\left[\iota(t)\!\!\left\| \nabla_{x'} \!\log P_{t\mid 0}\!(x'\!\!\mid \!\!x) \!- \!s_{\theta}(x, t) \right\|^2 \!\right] \!\!\mathrm{d}t,
\end{equation}

\textbf{Sampling.}
If we discretize the reverse process, initialize $x_T\sim \pi$ and replace $\nabla_x \log P_t(x)$ with the score-based model $s_{\theta}(x,t)$, we can generate samples with a Markov chain with $T$ steps:
\begin{equation}
x_{t-1}=x_t + [b(x_t,t)-r^2(t)s_{\theta}(x_t,t)]\Delta t +r(t)\sqrt{|\Delta t|}z_t,
\end{equation}
where $\Delta_t$ is a small enough constant and $z_t\sim\mathcal{N}(0,\mathbf{I})$. Most existing diffusion models generate samples following the Markov chain \cite{NEURIPS2020_4c5bcfec,score_matching_song2019generative,DDIM_song2020denoising} and a common expression for the joint distribution of the reverse outputs is
\begin{equation}\label{eqn:probability_diffusion}
    P_{\theta}(x_{0:T})=\pi(x_T)\prod_{t=1}^TP_{\theta}(x_{t-1}\mid x_t),
\end{equation}
where \(P_{\theta}(x_{t-1}\mid x_t)=\mathcal{N}(x_{t-1};\mu_{\theta}(x_t,t),\Sigma_{\theta}(x_t,t))\). 

We have the following lemma, which shows that the reverse KL divergence between the diffusion model and the nominal distribution is bounded. 

\begin{lemma}\label{KL_bound_diffusion}
    Given the above assumptions \ref{sec:assumptions}, if the score-matching loss satisfy $J(\theta, S)\leq \epsilon$, the output distribution of the diffusion model $P_{\theta}$ satisfies:
        \begin{equation}
            D_{\mathrm{KL}}(P_0 \,\|\, P_{\theta}) \leq \epsilon + D_{\mathrm{KL}}(P_T \,\|\, \pi) + C_1,
        \end{equation}
where $P_T$ is the output distribution of the forward process and $P_T\approx \pi$ by the design of diffusion models, and $C_1$ is a constant from approximating the score-matching loss.
\end{lemma}
Note that the \kl-divergence $D_{\mathrm{KL}}(P_0|| P_{\theta})$ in Lemma \ref{KL_bound_diffusion} is not the \kl-divergence $D_{\mathrm{KL}}(P_{\theta}||P_0)$ commonly used in \kldro. The former \kl-divergence allows $P_{\theta}$ to have broader support space than $P_0$ ($P_0\ll P$). By Lemma \ref{KL_bound_diffusion}, if we find an adversarial distribution $P_{\theta}$ by \eqref{eqn:objective}, $P_{\theta}$ also stays close enough to the training distribution $P_{0}$ through a \kl-divergence depending on the budget $\epsilon$. Therefore, the constraint in \eqref{eqn:objective} can define a probabilistic ambiguity set for \dro.

\section{Details of \dagalg in \ouralg Algorithm \ref{DAG-DRL}}\label{sec:algo_detail}

\subsection{Objective Derivation by Policy Gradient}\label{sec:policy_gradient}

Let $x_{0:T}$ be the output vector of each step in the backward process of the diffusion model. Denote $P_{0:T,\theta}$ as the joint distribution of $x_{0:T}$.   Since $x_0\sim P_{\theta}$, we can express the first term of the Lagrangian-relaxed objective as
\begin{equation}
    \mathbb{E}_{x\sim P_{\theta}}[f(w,x)]=  \mathbb{E}_{x_{0:T}\sim P_{0:T,\theta}}[f(w,x_0)].
\end{equation}

Then, the gradient of the Lagrangian-relaxed objective can be expressed as
\begin{equation}\label{eqn:policy_gradient}
  \begin{split}
     &\nabla_{\theta}\left(\mathbb{E}_{x\sim P_{\theta}}[f(w,x)]-\mu J(\theta, S_0)\right) \\
=& \nabla_{\theta}\left(\mathbb{E}_{x_{0:T}\sim P_{0:T,\theta}}[f(w,x_0)]-\mu J(\theta, S_0)\right)\\
=& \int_{x_{0:T}} f(w,x_0)\nabla_{\theta} P_{0:T,\theta}(x_{0:T})\mathrm{d}{x_{0:T}}-\mu \nabla_{\theta}J(\theta,S_0)\\
=& \int_{x_{0:T}} P_{0:T,\theta}(x_{0:T})f(w,x_0)\nabla_{\theta} \ln P_{0:T,\theta}(x_{0:T})\mathrm{d}x_{0:T}-\mu \nabla_{\theta}J(\theta,S_0)\\
=&\mathbb{E}_{x_{0:T}\sim P_{0:T,\theta}}\left[f(w,x_0)\nabla_{\theta} \ln P_{0:T,\theta}(x_{0:T})\right]-\mu \nabla_{\theta}J(\theta,S_0),
  \end{split}
\end{equation}
The first term $\mathbb{E}_{x_{0:T}\sim P_{0:T,\theta}}$ can be calculated by empirical mean $\hat{\mathbb{E}}_{x_{0:T}\sim P_{0:T,\theta}}$ based on the examples sampled from $P_{0:T,\theta}$. 
Thus, we can equivalently implement the gradient ascent by optimizing the objective in \eqref{eqn:vpg}:
\begin{equation}
  \max_{\theta} \hat{\mathbb{E}}_{x_{0:T}\in P_{0:T,\theta}}[
    \ln P_{0:T,\theta}(x_{0:T}) \cdot f(w,x_0)]-\mu\cdot J(\theta, S_0),
\end{equation}

Next, we derive the expression for 
$\ln P_{0:T,\theta}(x_{0:T})$.
Consider a discrete-time diffusion backward process as
\begin{equation}
x_{t-1} = \mu_{\theta}(x_t, t)+\sigma_t w_t,
\end{equation}
where $w_t$ is a standard multi-dimensional Gaussian variable. Given $\theta$, the conditional probability at each step $t$ is 
\begin{equation}
    P_{t-1,\theta}(x_{t-1}\mid x_t)=\mathcal{N}(x_{t-1},\mu_{\theta}(x_t, t),\sigma_t^2 \mathbf{I}),
\end{equation}
We can explicitly express the joint distribution $P_{0:T,\theta}$ of $x_{0:T}$ for all $T$ backward steps:
\begin{equation}
    P_{0:T,\theta}(x_{0:T})=P(x_T)\prod_{t=1}^T P_{t-1,\theta}(x_{t-1}\mid x_t)=\frac{1}{\sqrt{2\pi}}  e^{-\frac{\|x_T\|^2}{2}}\cdot \frac{1}{\sqrt{2\pi}\sigma_t}e^{-\sum_{t=1}^T\frac{\|x_{t-1}-\mu_{\theta}(x_t,t)\|^2}{2\sigma_t^2}},
\end{equation}
Thus, we can get the expression of $\ln P_{0:T,\theta}(x_{0:T})$ as
\begin{equation}
  \ln P_{0:T,\theta}(x_{0:T})=-\sum_{t=1}^T\|x_{t-1}-\mu_{\theta}(x_t,t)\|^2/(2\sigma_t^2)+C_2,
\end{equation}
where $C_2=-\ln(2\pi \sigma_t)-\frac{\|x_T\|^2}{2}$.

\subsection{Objective Derivation by Proximal Policy Optimization}\label{sec:ppo}
In this PPO method, we convert the first term of \eqref{eqn:lagrangian_objective} as a PPO-like objective given a reference diffusion model $P_{\theta_0}$, i.e.
\begin{equation}
\begin{split}
    \mathbb{E}_{x\sim P_{\theta}}[f(w,x)]&=  \mathbb{E}_{x_{0:T}\sim P_{0:T,\theta}}[f(w,x_0)]\\
    & = \mathbb{E}_{x_{0:T}\sim P_{0:T,\theta_0}}\left[\frac{P_{0:T,\theta }(x_{0:T})}{P_{0:T,\theta_0}(x_{0:T})}f(w,x_0)\right],
\end{split}
\end{equation}
where the probability ratio is \[r_{\theta}(x_{0:T})=\frac{P_{0:T,\theta }(x_{0:T})}{P_{0:T,\theta_0}(x_{0:T})}=\exp\left\{-\sum_{t=1}^T(\frac{\|x_{t-1}-\mu_{\theta}(x_t,t)\|^2}{2\sigma_{t}^2}-\frac{\|x_{t-1}-\mu_{\theta_0}(x_t,t)\|^2}{2\sigma_{t}^2})\right\}\] given the joint probability expression.
The expectation can be approximated by the empirical mean based on the examples sampled by $P_{0:T,\theta_0}$. Like PPO, we apply clipping on the ratio to avoid overly-large updates, and we can get the objective in Eqn.~\eqref{eqn:ppo}:
\begin{equation}
  \max_{\theta} \hat{\mathbb{E}}_{ P_{0:T,\theta_0}(x_{0:T})}[
   \min(r_{\theta}(x_{0:T})f(w,x_0)), \mathrm{clip}(r_{\theta}(x_{0:T}),1-\kappa,1+\kappa)\cdot f(w,x_0))]-\mu\cdot J(\theta, S_0),
\end{equation}
where $\kappa\in (0,1)$ is the clipping parameter.

\section{Theorem Proofs}\label{sec:proofs}

\subsection{Proof of Lemma \ref{KL_bound_diffusion}}
\subsubsection{Bound of the Score-Matching Loss by KL Divergence}

Lemma \ref{KL_bound_diffusion} is a result based on the conclusion of Theorem 1 in \cite{MLE_score_based_song2021maximum}. For completeness, we give the full proof of Theorem 1 in \cite{MLE_score_based_song2021maximum} (Lemma \ref{lma:KLbound_SMloss}), and then we derive the concrete expressions for the constants in Lemma \ref{KL_bound_diffusion}.

\paragraph{Assumptions}\label{sec:assumptions}
We make the following assumptions for all the lemmas and theorems in the paper.

\begin{enumerate}
 \item $P_0(x)$ is a density function with continuous second-order derivatives and $\mathbb{E}_{x\sim P_0}\left[\|x\|_2^2\right]<\infty$.
 \item The prior distribution $\pi(x)$ is a density function with continuous second-order derivatives and $\mathbb{E}_{x\sim \pi}\left[\|x\|_2^2\right]<\infty$.
 \item $\forall t\in [0,T]$: $f(\cdot, t)$ is a function with continuous first order derivatives. $\exists C>0, \forall x\in \mathcal{R}^d, t\in [0,T]$: $\|b(x,t)\|_2\leq C(1+\|x\|_2)$
 \item $\exists C>0, \forall x,y\in \mathcal{R}^d: \|b(x,t)-b(y,t)\|_2\leq C\|x-y\|_2$.
 \item $g$ is a continuous function and $\forall t\in [0,T], |r(t)|>0$.
 \item For any open bounded set $\mathcal{O}$, $\int_{t=0}^T\int_{\mathcal{O}}\|P_t(x)\|_2^2+dr(t)^2\|\nabla_xP_t(x)\|_2^2\mathrm{d}x\mathrm{d}t$.
 \item $\exists C>0,\forall x \in \mathcal{R}^d, t\in [0,T]:\|\nabla_x\log P_t(x)\|_2\leq C(1+\|x\|_2)$.
 \item $\exists C>0,\forall x,y \in \mathcal{R}^d:\|\nabla_x\log P_t(x)-\nabla_x\log P_t(y)\|_2\leq C(\|x-y\|_2)$.
 \item $\exists C>0,\forall x \in \mathcal{R}^d, t\in [0,T]:\|s_{\theta}(x,t)\|_2\leq C(1+\|x\|_2)$.
\item $\exists C>0,\forall x,y \in \mathcal{R}^d:\|s_{\theta}(x,t)-s_{\theta}(y,t)\|_2\leq C(\|x-y\|_2)$.
\item Novikov's condition: $\mathbb{E}\left[\exp(\frac{1}{2}\int_{t=0}^T\|\nabla_x\log P_t(x)-s_{\theta}(x,t)\|^2_2\mathrm{d}t)\right]<\infty$.
 \end{enumerate}

\begin{lemma}\label{lma:KLbound_SMloss}
    Let $P_0(x)$ be the underlying data distribution, $\pi(x)$ be a known prior distribution, and $P_{\theta}$ be the marginal distribution of $\hat{x}_{\theta}(0)$, the output of reverse-time SDE defined as below:
    \begin{equation}\label{eqn:denoising_SDE}
   \mathrm{d}\hat{x}=[b(\hat{x},t)-r(t)^2s_{\theta}(\hat{x},t)]\mathrm{d}t+r(t)\mathrm{d}\bar{w},\; \hat{x}_{\theta}(T)\sim \pi,
    \end{equation}
    With the assumptions in Section \ref{sec:assumptions}, we have
    \begin{equation}
        D_{\mathrm{KL}}(P_0\|P_{\theta})\leq J_{SM}(\theta,r(\cdot)^2)+D_{\mathrm{KL}}(P_T\|\pi),
    \end{equation}
    where $J_{SM}(\theta,r(\cdot)^2)=\frac{1}{2}\int_{t=0}^T\mathbb{E}_{p_t(x)}\left[r(t)\|\nabla_x\log P_t(x)-s_{\theta}(x,t)\|_2^2\right]\mathrm{d}t.$
    \end{lemma}
\begin{proof}
We denote the path measure of the forward outputs $\{x_t\}_{t\in[0,T]}$ as $p$ and the path measure of the backward outputs $\{\hat{x}_{\theta,t}\}_{t\in[0,T]}$ as $q$. By assumptions 1- 5, 9, 10, both $p$ and $q$ are uniquely given by the forward and backward SDEs, respectively. Consider a Markov kernel $M(\{z_t\}_{t\in[0,T]},y):= \delta (z_0=y)$ given any Markov chain $\{z_t\}_{t\in[0,T]}$. Since $x_0\sim P_0$ and $\hat{x}_{\theta,0}\sim P_{\theta}$, we have 
\begin{equation}
    \int M(\{x_t\}_{t\in [0,T]},x) \mathrm{d}p(\{x_t\}_{t\in [0,T]})=P_0(x),
\end{equation}
\begin{equation}
    \int M(\{\hat{x}_{\theta,t}\}_{t\in [0,T]},x) \mathrm{d}q(\hat{x}_{\theta,t}\}_{t\in [0,T]})=P_{\theta}(x),
\end{equation}
Here, the Markov kernel $M$ essentially performs marginalization of path measures to obtain distributions at $t=0$. We can use the data processing inequality with this Markov kernel to obtain
\begin{equation}
\begin{split}
    &D_{\mathrm{KL}}(P_0\|P_{\theta})\\
    =&D_{\mathrm{KL}}\left( \int M(\{x_t\}_{t\in [0,T]},x) \mathrm{d}p(\{x_t\}_{t\in [0,T]})\| \int M(\{\hat{x}_{\theta,t}\}_{t\in [0,T]},x) \mathrm{d}q(\hat{x}_{\theta,t}\}_{t\in [0,T]})\right)\\
    \leq &D_{\mathrm{KL}}(p,q),
\end{split}
\end{equation}
Since $x_T\sim P_T$ and $\hat{x}_{\theta, T}\sim \pi$. Leveraging the chain rule of KL divergence, we have
\begin{equation}
\begin{split}
 D_{\mathrm{KL}}(p,q)&=\mathbb{E}_p\left[\log(\frac{p(x_{1:T}\mid x_T)P_T(x_T)}{q((\hat{x}_{\theta,1:T}\mid \hat{x}_{\theta,T})\pi(\hat{x}_{\theta,T})})\right]
\\
 &=  D_{\mathrm{KL}}(P_T\|\pi)+\mathbb{E}_{z\sim P_T}\left[D_{\mathrm{KL}}(p(\cdot\mid x_T=z)\|q(\cdot\mid \hat{x}_{\theta,T}=z)) \right],
 \end{split}
\end{equation}
Under assumptions 1,3-8, the SDE in Eqn. \eqref{eqn:forward} has a corresponding reverse-time SDE as
\begin{equation}\label{eqn:reverse-time_SDE}
    \mathrm{d}x=[b(x,t)-r(t)^2\nabla_x\log P_t(x)]+r(t)\mathrm{d}\bar{w},
\end{equation}
Since Eqn. \eqref{eqn:reverse-time_SDE} is the time reversal of Eqn. \eqref{eqn:forward}, they share the same path measure $p$. Thus, $\mathbb{E}_{z\sim P_T}\left[D_{\mathrm{KL}}(p(\cdot\mid x_T=z)\|q(\cdot\mid \hat{x}_{\theta,T}=z)) \right]$ can be viewed as the KL divergence between the path measures induced by the two SDEs in Eqn. \eqref{eqn:forward} and Eqn. \eqref{eqn:denoising_SDE} with the same starting points $x_T=\hat{x}_{\theta,T}=z$.

The KL divergence between two SDEs with shared diffusion coefficients and starting points exists under assumptions 7,-11, and can be bounded by the Girsanov theorem
\begin{equation}
\begin{split}
&D_{\mathrm{KL}}(p(\cdot\mid x_T=z)\|q(\cdot\mid \hat{x}_{\theta,T}=z))=\mathbb{E}_p\left[\log\frac{\mathrm{d}p}{\mathrm{d}q} \right]\\
=& \mathbb{E}_p\left[ \int_{t=0}^Tr(t)(\nabla_x\log P_t(x)-s_{\theta}(x,t))\mathrm{d}\bar{w}_t+\frac{1}{2}\int_{t=0}^Tr(t)^2\|\nabla_x\log P_t(x)-s_{\theta}(x,t)\|_2^2\mathrm{d}t\right]\\
=&\mathbb{E}_p\left[\frac{1}{2}\int_{t=0}^Tr(t)^2\|\nabla_x\log P_t(x)-s_{\theta}(x,t)\|_2^2\mathrm{d}t\right]\\
=& \frac{1}{2}\int_{t=0}^T\mathbb{E}_{P_t(x)}\left[r(t)^2\|\nabla_x\log P_t(x)-s_{\theta}(x,t)\|_2^2\right]\mathrm{d}t=J_{SM}(\theta,r(\cdot)^2),
\end{split}
\end{equation}
where the second equality holds by Girsanov Theorem II, and the third equality holds because $Y_s=\int_{t=0}^s r(t)(\nabla_x\log P_t(x)-s_{\theta}(x,t))\mathrm{d}\bar{w}_t$ is a continuous-time Martingale process ($\mathbb{E}[Y_s\mid {Y_{\tau, \tau\leq s'}}]=Y_{s'}, \forall s'\leq s$) and we have $\mathbb{E}[Y_s-Y_{s'}]=0, \forall s'<s$.
\end{proof}

\subsubsection{Proof of Lemma \ref{KL_bound_diffusion}}
\textbf{Lemma \ref{KL_bound_diffusion}}. 
\textit{Given the assumptions in Appendix \ref{sec:assumptions}, if the score-matching loss satisfies $J(\theta, S)\leq \epsilon$, the output distribution of the diffusion model $P_{\theta}$ satisfies
\[
D_{\mathrm{KL}}(P_0|| P_{\theta})\leq \epsilon + D_{\mathrm{KL}}(P_T||\pi)+C_1,
\]
where $P_T$ is the output distribution of the forward process, $\pi$ is a prior distribution of the diffusion model and $P_T\approx \pi$ by the design of diffusion models, and $C_1$ is a constant that does not rely on $\theta$.}

The score matching loss $J(\theta, S_0)$ in \eqref{eqn:score_match} with $\iota(t)=r(t)^2$ is actually the denoising score matching loss $J_{DSM}(\theta, r(\cdot)^2)$, i.e.
\begin{equation}
    J(\theta, S_0)=J_{DSM}(\theta, r(\cdot)^2)=\frac{1}{2}\int_0^T \mathbb{E}_{x_0\sim P_0} \mathbb{E}_{x_t\sim P_{t\mid 0}} \left[r(t)^2\left\| \nabla_{x_t} \log P_{t\mid 0}(x_t \mid x_0) - s_{\theta}(x_t, t) \right\|^2 \right] \mathrm{d}t,
\end{equation}
which is used to compute the original score matching loss $J_{SM}(\theta,r(\cdot)^2)$ given a dataset $S_0$. The gap between $J_{DSM}(\theta, r(\cdot)^2)$ and $J_{SM}(\theta,r(\cdot)^2)$ is a constant $C_1$ that does not depend on $\theta$, which is shown as below.

The difference is expressed as
\begin{equation}
\begin{aligned}
   &J_{SM}(\theta, r(\cdot)^2) - J_{DSM}(\theta,r(\cdot)^2)\\
   =& \frac{1}{2}\int_{t=0}^T\mathbb{E}_{P_{0,t}(x_0,x_t)}\left[r(t)^2\left(\|\nabla_{x_t}\log P_t(x_t)-s_{\theta}(x_t,t)\|_2^2-\|\nabla_{x_t} \log P_{t\mid 0}(x_t \mid x_0)-s_{\theta}(x_t,t)\|_2^2\right)\right]\mathrm{d}t\\
   =& \int_{t=0}^T r(t)^2\left(\underbrace{\mathbb{E}_{P_{0,t}(x_0,x_t)}\left[-\left\langle s_{\theta}(x_t,t), \nabla_{x_t}\log P_t(x_t)+\nabla_{x_t} \log P_{t\mid 0}(x_t \mid x_0)\right\rangle\right]}_{\mathrm{(1)}}+C_1'(x_0,x_t)\right)\mathrm{d}t ,
   \end{aligned}
\end{equation}
where $P_{0,t}$ is the joint distribution of $x_0$ and $x_t$, and $C_1'(x_0,x_t)=$\\$\mathbb{E}_{P_{0,t}(x_0,x_t)}\left[\frac{1}{2}\|\nabla_{x_t}\log P_t(x_t)\|_2^2-\frac{1}{2}\|\nabla_{x_t} \log P_{t\mid 0}(x_t \mid x_0)\|_2^2\right]$. 

The first term $(1)$ is zero because
\begin{equation}
\begin{split}
 &\mathbb{E}_{P_{0,t}(x_0,x_t)}\left[\left\langle s_{\theta}(x_t,t), \nabla_{x_t}\log P_t(x_t)\right\rangle\right]= \mathbb{E}_{P_{t}(x_t)}\left[\left\langle s_{\theta}(x_t,t), \nabla_{x_t}\log P_t(x_t)\right\rangle\right]\\
 =&\int_{x_t}\left\langle s_{\theta}(x_t,t), \frac{1}{P_t(x_t)}\nabla_{x_t} P_t(x_t)\right\rangle P_t(x_t)\mathrm{d}x_t\\
 =&\int_{x_t}\left\langle s_{\theta}(x_t,t), \nabla_{x_t} \int _{x_0}P_{t\mid 0}(x_t\mid x_0)P_0(x_0)\mathrm{d}x_0\right\rangle \mathrm{d}x_t\\
 =&\int_{x_t}\left\langle s_{\theta}(x_t,t), \int _{x_0}  P_{t\mid 0}(x_t\mid x_0)P_0(x_0)\nabla_{x_t}\log (P_t(x_t\mid x_0)) \mathrm{d}x_0\right\rangle \mathrm{d}x_t\\
 =&\int_{x_0,x_t}P_{0,t}(x_0,x_t)\left\langle s_{\theta}(x_t,t),  \nabla_{x_t}\log (P_{t\mid 0}(x_t\mid x_0)) \right\rangle \mathrm{d}x_0\mathrm{d}x_t\\
 =&\mathbb{E}_{P_{0,t}(x_0,x_t)}\left[ \left\langle s_{\theta}(x_t,t),  \nabla_{x_t}\log P_{t\mid 0}(x_t\mid x_0) \right\rangle\right],
 \end{split}
\end{equation}
Thus, we can bound the gap between $J_{DSM}(\theta, r(\cdot)^2)$ and $J_{SM}(\theta,r(\cdot)^2)$ as 
\begin{equation}
\begin{split}
   C_1=J_{SM}(\theta, r(\cdot)^2) - J_{DSM}(\theta,r(\cdot)^2)
   = \int_{t=0}^T r(t)^2 C_1'(x_0,x_t)\mathrm{d}t ,
   \end{split}
\end{equation}
where $C_1'(x_0,x_t)=\mathbb{E}_{P_{0,t}(x_0,x_t)}\left[\frac{1}{2}\|\nabla_{x_t}\log P_t(x_t)\|_2^2-\frac{1}{2}\|\nabla_{x_t} \log P_{t\mid 0}(x_t \mid x_0)\|_2^2\right]$.

Therefore, if $J(\theta,S_0)=J_{DSM}(\theta,r(\cdot)^2)\leq \epsilon$, by Lemma \ref{lma:KLbound_SMloss}, we have
 \begin{equation}\label{eqn:bound_KL_DSM}
 \begin{split}
        D_{\mathrm{KL}}(P_0\|P_{\theta})\leq& J_{SM}(\theta,r(\cdot)^2)+D_{\mathrm{KL}}(P_T\|\pi)\\
     = &  J_{DSM}(\theta,r(\cdot)^2)+D_{\mathrm{KL}}(P_T\|\pi)+C_1\\
     \leq &\epsilon+D_{\mathrm{KL}}(P_T\|\pi) +C_1,
    \end{split}
\end{equation}
which completes the proof.

\subsection{Proof of Theorem \ref{thm:analysis}}\label{sec:proofconvergencemax}
\textbf{Theorem 1}.
  \textit{ Let $\theta^*$ be the optimal diffusion parameter that solves the inner maximization \eqref{eqn:objective} given a variable $w$. 
       If the expected score-matching loss is bounded as $J(\theta)\leq \bar{J}$ and the step size is chosen as $\eta\sim\mathcal{O}(\frac{1}{\sqrt{K}})$, the inner maximization error holds that   
       \begin{equation}
         \Delta':= \mathbb{E}_{x\sim P_{\theta^*}}[f(w,x)]-\mathbb{E}_k \mathbb{E}_{x\sim P_{\theta_k}}[f(w,x)]\leq \frac{1}{\sqrt{K}} \max\{\epsilon, \bar{J}\}\|\mu^{(1)}\|,
       \end{equation}
       where the outer expectation is taken over the randomness of output selection.
       In addition, the \kl-divergence with respect to the nominal distribution is bounded as
      \(\mathbb{E}_k[D_{\mathrm{KL}}(P_0||P_{\theta_k})]\leq \epsilon + \frac{\max\{\epsilon, \bar{J}\}|\mu_C-\mu^{(1)}|}{\sqrt{K} (\mu_C-\mu^*)}+ D_{\mathrm{KL}}(P_T||\pi)+C_1 \),
    where $\mu_C>\mu^*$ and $C_1$ are constants, $P_T$ is the output distribution of the forward process, and $\pi$ is the initial distribution of the reverse process.  }

\subsubsection{Convergence of Dual Gradient Descent}
To prove Theorem \ref{thm:analysis}, we need the convergence analysis of a general dual gradient descent in the following lemma. 
\begin{lemma}\label{lma:dual_convergence}
Assume that the dual variable is updated following
\[
\mu_{k+1}=\max\{\mu_k-\eta\cdot b_k,0\},
\]
where $\eta>0$ and $b_k$ has the same dimension as $\mu$ and $\max\{\cdot,0\}$ is an element-wise non-negativity operation.
With $\eta$ as the step size in dual gradient descent, given any $\mu>0$, we have
\begin{equation}
  \frac{1}{K}\sum_{k=1}^K\left\langle\mu_k-\mu, b_k\right\rangle\leq  \eta \frac{1}{K}\sum_{k=1}^K \|b_k\|^2 + \frac{1}{2K\eta}\|\mu-\mu^{(1)}\|^2,
\end{equation}
\end{lemma}
\begin{proof}
For any dimension $j$ such that $\mu_{k,j}\geq \eta b_{k,j}$, we have $\mu_{k+1,j}=\mu_{k,j}-\eta\cdot b_{k,j}$, so it holds for any $\mu>0$ that
\begin{equation}\label{eqn:optimality_condition}
    \left(b_{k,j}+\frac{1}{\eta}(\mu_{k+1,j}-\mu_{k,j})\right)(\mu_j-\mu_{k+1,j})= 0,
\end{equation}
For any dimension $j$ such that  $\mu_{k,j}< \eta b_{k,j}$, we have $\mu_{k+1,j}=0$, and it holds for any $\mu>0$ that  
\begin{equation}\label{eqn:optimality_condition_2}
   \left(b_{k,j}+\frac{1}{\eta}(\mu_{k+1,j}-\mu_{k,j})\right) (\mu_j-\mu_{k+1,j})=(b_{k,j}-\frac{\mu_{k,j}}{\eta})\mu_j\geq (b_{k,j}-\frac{\eta b_{k,j}}{\eta})\mu_j =0,
\end{equation}
Combing \eqref{eqn:optimality_condition} and \eqref{eqn:optimality_condition_2} and we have for any $\mu>0$ that
\begin{equation}\label{eqn:optimality_condition_all}
    \left(b_{k}+\frac{1}{\eta}(\mu_{k+1}-\mu_{k})\right)^\top(\mu-\mu_{k+1})\geq 0,
\end{equation}

Therefore, it holds for any $k\in[K]$ and $\mu>0$ that
\begin{equation}
\begin{split}
  (\mu_k-\mu)^\top b_k =& (\mu_k-\mu_{k+1})^\top b_k+(\mu_{k+1}-\mu)^\top b_k\\
  \leq &(\mu_k-\mu_{k+1})^\top b_k + \frac{1}{\eta}(\mu_{k+1}-\mu_k)^\top (\mu-\mu_{k+1})\\
  = & (\mu_k-\mu_{k+1})^\top b_k + \frac{1}{2\eta}\left(\|\mu-\mu_k\|^2-\|\mu-\mu_{k+1}\|^2-\|\mu_{k+1}-\mu_{k}\|^2\right)\\
  \leq& \eta\|b_k\|^2+\frac{1}{2\eta}\left(\|\mu-\mu_k\|^2-\|\mu-\mu_{k+1}\|^2\right),
\end{split}
\end{equation}
where the first inequality holds by \eqref{eqn:optimality_condition_all}, the second equality holds by three-point property, and the last inequality holds because $\|a\|^2+\|b\|^2\geq 2a^\top b$.

Taking the sum over $k\in [K]$, we have
\begin{equation}
    \sum_{k=1}^K\left\langle\mu_k-\mu, b_k\right\rangle\leq \eta\sum_{k=1}^K\|b_k\|^2+\frac{1}{2\eta}\left(\|\mu-\mu_{1}\|^2\right),
\end{equation}
which proves Lemma \ref{lma:dual_convergence}.
\end{proof}

\subsubsection{Bound of Expected Objective in Theorem \ref{thm:analysis}}
\begin{proof}
    Denote $F(\theta)=\mathbb{E}_{x\sim P_{\theta}}[f(w,x)]$ and the optimal parameter to solve the inner maximization of \eqref{eqn:objective} is $\theta^*$. Define the dual problem of the inner maximization of \eqref{eqn:objective} as
    \begin{equation}
        Q(\mu)= \max_{\theta} F(\theta)+\mu( \epsilon- J(\theta, S_0)),
    \end{equation}
    Given any $\mu>0$, it holds be weak duality that
    \begin{equation}
        F(\theta^*)\leq F(\theta^*)+\mu( \epsilon- J(\theta^*, S_0)) \leq Q(\mu),
    \end{equation}
    where the second inequality holds because $Q(\mu)$ maximizes $F(\theta)+\mu( \epsilon- J(\theta, S_0))$ over $\theta$. 
    Thus, the average gap between the expected loss of $\theta^*$ and $\theta_k$ is
    \begin{equation}
    \begin{split}
    \frac{1}{K}\sum_{k=1}^K (F(\theta^*)-F(\theta_k))&\leq  \frac{1}{T}\sum_{t=1}^T (Q(\mu_k)-F(\theta_k))\\
    &=\frac{1}{K}\sum_{k=1}^K\mu_k( \epsilon- J(\theta_k, S_0))\\
    &\leq  \eta \frac{1}{K}\sum_{k=1}^K (\epsilon- J(\theta_k, S_0))^2 +  \frac{1}{2\eta} \frac{1}{K}\|\mu^{(1)}\|^2\\
    &\leq  \eta (\max\{\epsilon, \bar{J}\})^2 +  \frac{1}{\eta} \frac{1}{K}\|\mu^{(1)}\|^2,
    \end{split}
    \end{equation} 
    where the equality holds because $\theta_k=\arg\max_{\theta} F(\theta)+\mu_k( \epsilon- J(\theta, S_0))$ by \eqref{eqn:lagrangian_objective} and so $Q(\mu_k)=F(\theta_k)+\mu_k( \epsilon- J(\theta_k, S_0))$, the second inequality holds by Lemma \ref{lma:dual_convergence} with the choice of $\mu=0$, and the last inequality holds by the assumption $J(\theta_k, S_0)\leq \bar{J}$.

    Choosing $\eta=\frac{\mu^{(1)}}{\sqrt{K}\max\{\epsilon,\bar{J}\}}$, it holds that
\begin{equation}
     \frac{1}{K}\sum_{k=1}^K (F(\theta^*)-F(\theta_k)) \leq \frac{1}{\sqrt{K}} \max\{\epsilon, \bar{J}\}\mu^{(1)},
\end{equation}
which proves the bound of the expected loss given the uniformly selected $k\in[K]$.

\subsubsection{Bound of KL Divergence in Theorem \ref{thm:analysis}}
For iteration $k$, denote $b(\theta)=\epsilon-J(\theta,S_0)$  and $b_k=\epsilon-J(\theta_k,S_0)$. Denote the constraint violation on the score matching loss at round $k$ as $v_k=J(\theta_k,S_0)-\epsilon$.
Denote the optimal dual variable as $\mu^*=\arg\min_{\mu} Q(\mu)$.
Choose a dual variable $\mu_C>\mu^*$, we have the following decomposition.
\begin{equation}\label{eqn:decomposition}
    \sum_{k=1}^K (\mu_k-\mu_C)b_k= \sum_{k=1}^K \mu_k\cdot b_k+\sum_{k=1}^K (-\mu_C)\cdot b_k,
\end{equation}
For the first term, we have
\begin{equation}
\begin{split}
\sum_{k=1}^K \mu_k\cdot b_k &=\sum_{k=1}^K F(\theta_k) + \mu_k\cdot b_k -  F(\theta_k)= \sum_{k=1}^K Q(\mu_k) -  F(\theta_k)\\
&\geq K Q(\mu^*)- \sum_{k=1}^K F(\theta_k)\\
&\geq K Q(\mu^*)- \sum_{k=1}^K \max_{\theta\in \{\theta\mid J(\theta,S_0)\leq \epsilon+v_k\}}F(\theta)\\
&\geq K Q(\mu^*)- \sum_{k=1}^K \max_{\theta}(F(\theta)+\mu^* (\epsilon+v_k-J(\theta,S_0)))\\
&= K Q(\mu^*)-KQ(\mu^*) -\mu^*\sum_{k=1}^K v_k = -\mu^*\sum_{k=1}^K v_k,
\end{split}
\end{equation}
where the first inequality is because $\mu^*$ minimizes $Q(\mu_k)$, the second inequality holds because $\theta_k\in\{\theta\mid J(\theta,S_0)\leq \epsilon+v_k\}$, the third inequality holds by weak duality for $\mu^*$. Continuing with \eqref{eqn:decomposition}, we have
\begin{equation}\label{eqn:convergence_inner_proof_5}
    \sum_{k=1}^K (\mu_k-\mu_C)b_k\geq \sum_{k=1}^K (-\mu^*v_k +\mu_C (J(\theta_k,S_0)-\epsilon))= (-\mu^* +\mu_C )\sum_{k=1}^K v_k,
\end{equation}
where the equality holds because $v_k= J(\theta_k,S_0)-\epsilon$.

By Lemma \ref{lma:dual_convergence} with the choice of $\mu=\mu_C$, we have
\begin{equation}
 \frac{1}{K}\sum_{k=1}^K (\mu_k-\mu_C)b_k \leq \eta \frac{1}{K}\sum_{k=1}^K \|b_k\|^2 + \frac{1}{2K\eta}(\mu_C-\mu^{(1)})^2\leq \eta  (\max\{\epsilon,\bar{J}\})^2 + \frac{1}{2K\eta}(\mu_C-\mu^{(1)})^2,
\end{equation}
If the step size is chosen as $\eta=\frac{\mu^{(1)}}{\sqrt{K}\max\{\epsilon,\bar{J}\}}$, it holds that
\begin{equation}
 \frac{1}{K}\sum_{k=1}^K (\mu_k-\mu_C)b_k\leq \frac{1}{\sqrt{K}} \max\{\epsilon, \bar{J}\}\left(\mu^{(1)}+\frac{|\mu_C-\mu^{(1)}|^2}{2\mu^{(1)}}\right),
\end{equation}
Since $\mu_C$ is larger than $\mu^*$, by \eqref{eqn:convergence_inner_proof_5}, we have 
\begin{equation}
 \frac{1}{K}\sum_{k=1}^K v_k\leq \frac{1}{K(\mu_C-\mu^*)} \sum_{k=1}^K (\mu_k-\mu_C)b_k\leq \frac{\max\{\epsilon, \bar{J}\}\left(\mu^{(1)}+\frac{|\mu_C-\mu^{(1)}|
 ^2}{2\mu^{(1)}}\right)}{\sqrt{K} (\mu_C-\mu^*)}=\frac{C_4}{\sqrt{K}},
\end{equation}
which means $\frac{1}{K}\sum_{k=1}^K J(\theta_k,S_0)\leq \epsilon + \frac{C_4}{\sqrt{K}}$.

Since $J(\theta_k,S_0)$ is the denoising score matching loss $J_{DSM}(\theta,r(\cdot)^2)$, by Lemma \ref{KL_bound_diffusion}, we complete the proof by 
\begin{equation}
\begin{split}
   \frac{1}{K}\sum_{k=1}^K D_{\mathrm{KL}}(P_0,P_{\theta_k})&\leq  \frac{1}{K}\sum_{k=1}^K J(\theta_k,S_0)+D_{\mathrm{KL}}(P_T\|\pi)+C_1\\
   &\leq \epsilon + \frac{C_4}{\sqrt{K}}+ D_{\mathrm{KL}}(P_T||\pi)+C_1,
\end{split}
\end{equation}
where $ C_1
   = \int_{t=0}^T r(t)^2 \left(\mathbb{E}_{P_{0,t}(x_0,x_t)}\left[\frac{1}{2}\|\nabla_{x_t}\log P_t(x_t)\|_2^2-\frac{1}{2}\|\nabla_{x_t} \log P_{t\mid 0}(x_t \mid x_0)\|_2^2\right]\right)\mathrm{d}t$,  $C_4=\frac{\max\{\epsilon, \bar{J}\}\left(\mu^{(1)}+\frac{|\mu_C-\mu^{(1)}|
 ^2}{2\mu^{(1)}}\right)}{(\mu_C-\mu^*)}$ with $\mu_C>\mu^*$.
\end{proof}

\subsection{Proof of Theorem \ref{thm:convergence_minmax}}
\label{sec:proofconvergenceminmax}
\textbf{Theorem \ref{thm:convergence_minmax}.}
\textit{
   Assume that the DRO objective $f(w,x)$ is $\beta-$smooth and $L-$Lipschitz with respect to $w$ and is upper bounded by $\bar{f}$. If each sampled dataset $S_j$ from diffusion model has $n$ samples and the step size for minimization is chosen as $\lambda\sim\mathcal{O}(\sqrt{\frac{1}{\beta L^2 H}})$, then with probability $1-\delta, \delta\in(0,1)$, the average Moreau envelope of the optimal inner maximization function $\phi(w):= \max_{\theta}\mathbb{E}_{P_{\theta}}[f(w,x)]$ satisfies
   \begin{equation}
\begin{split}
\mathbb{E}\left[\|\nabla\phi_{\frac{1}{2\beta}}(w)\|^2\right]
\leq 4\beta \Delta'+\frac{V_1}{\sqrt{n}}+\frac{V_2}{\sqrt{H}},
\end{split}
\end{equation}
where $\Delta'$ is the error of inner maximization bounded in Theorem \ref{thm:analysis}, $V_1=8\beta \bar{f}\sqrt{\log(2/\delta)}$ and $V_2=4L\sqrt{(\phi_{\frac{1}{2\beta}}(w_{1})-\min_w\phi(w))\beta}$.
}
\begin{proof}
The proof of Theorem \ref{thm:convergence_minmax} follows the techniques of  \cite{nonconvex_nonconcave_mini_max_jin2020local} with the difference that the maximization oracle is an empirical approximation. 
The optimization variable $w$ is updated as $w_j=w_{j-1}-\lambda\cdot \nabla_w \mathbb{E}_{x\in S_j}[f(w_{j-1},x)]$.
Define $\phi(w):= \max_{\theta}\mathbb{E}_{P_{\theta}}[f(w,x)]$ as the inner maximization function and define $\phi_{\rho}(w):=\min_{w'}\phi(w')+\frac{1}{2\rho}\|w-w'\|^2$ as the Moreau envelop of $\phi$. 

Denote $\hat{w}_j=\arg\min_w \phi(w)+\beta \|w-w_j\|^2$ as the proximal point of the Moreau envelop $\phi_{\frac{1}{2\beta}}(w)$.  Since $f$ is $\beta-$smooth, we have
\begin{equation}\label{eqn:smoothness}
  f(\hat{w}_{j},x) \geq f(w_j,x)+\left\langle\nabla_wf(w_j,x),\hat{w}_{j}-w_j\right\rangle-\frac{\beta}{2}\|\hat{w}_{j}-w_j\|^2,
\end{equation}
Thus, it holds with probability at least $1-\delta, \delta\in (0,1)$ that
\begin{equation}\label{eqn:bound_phi}
\begin{split}
    \phi(\hat{w}_j)&\geq \mathbb{E}_{P_{\theta_j}}[f(\hat{w}_j,x)]\geq \mathbb{E}_{S_{j}}[f(\hat{w}_{j},x)] - \frac{\bar{f}\sqrt{\log(2/\delta)}}{\sqrt{n}}\\
    &\geq \mathbb{E}_{S_{j}}[f(w_j,x)]+\left\langle g(w_j,S_j),\hat{w}_{j}-w_j\right\rangle-\frac{\beta}{2}\|\hat{w}_{j}-w_j\|^2- \frac{\bar{f}\sqrt{\log(2/\delta)}}{\sqrt{n}}\\
    &\geq \mathbb{E}_{P_{\theta_j}}[f(w_j,x)]+\left\langle g(w_j,S_j),\hat{w}_{j}-w_j\right\rangle-\frac{\beta}{2}\|\hat{w}_{j}-w_j\|^2- \frac{2\bar{f}\sqrt{\log(2/\delta)}}{\sqrt{n}}\\
    &\geq \phi(w_j)-\Delta'_j +\left\langle g(w_j,S_j),\hat{w}_{j}-w_j\right\rangle-\frac{\beta}{2}\|\hat{w}_{j}-w_j\|^2- \frac{2\bar{f}\sqrt{\log(2/\delta)}}{\sqrt{n}},
\end{split}
\end{equation}
where the second inequality holds by applying McDiarmid's inequality on $\mathbb{E}_{P_{\theta_j}}[f(\hat{w}_j,x)]$ and $\bar{f}$ is the upper bound of $f$, the third inequality holds by \eqref{eqn:smoothness} and $g(w_j,S_j)=\mathbb{E}_{S_j}[\nabla_w f(w_j,x)]$, the forth inequality holds by applying McDiarmid's inequality on $\mathbb{E}_{S_{j}}[f(w_j,x)]$, and the last inequality holds by Theorem \ref{thm:analysis}. 

Now, it holds that
\begin{equation}\label{eqn:proofconvergence_3}
\begin{split}
    &\phi_{\frac{1}{2\beta}}(w_{j+1})\\
    =&\min_{w'}\phi(w')+\beta\|w_{j+1}-w'\|^2\\
    \leq &\phi(\hat{w}_j)+\beta\|w_{j}-\lambda g(w_{j},S_j)-\hat{w}_j\|^2\\
    =&\phi(\hat{w}_j)+\beta\|w_{j}-\hat{w}_j\|^2+2\beta\lambda <g(w_{j},S_j),\hat{w}_j-w_j>+\lambda^2\beta\|g(w_{j},S_j)\|^2 \\
    \leq& \phi_{\frac{1}{2\beta}}(w_{j})+2\beta\lambda <g(w_{j},S_j),\hat{w}_j-w_j>+\lambda^2\beta\|g(w_{j},S_j)\|^2 \\
    \leq& \phi_{\frac{1}{2\beta}}(w_{j})+2\beta\lambda \left( \phi(\hat{w}_j)-\phi(w_j)+\Delta'+\frac{\beta}{2}\|\hat{w}_{j}-w_j\|^2+\frac{2\bar{f}\sqrt{\log(2/\delta)}}{\sqrt{n}}\right)+\lambda^2\beta L^2,
\end{split}
\end{equation}
where the last inequality holds by \eqref{eqn:bound_phi}.

Summing up inequality \eqref{eqn:proofconvergence_3} from $j=1$ to $j=H$, we have
\begin{equation}
\begin{aligned}
&\phi_{\frac{1}{2\beta}}(w_{j+1})\\
\leq &\phi_{\frac{1}{2\beta}}(w_{1})+2\beta\lambda \sum_{j=1}^H\left( \phi(\hat{w}_j)-\phi(w_j)+\Delta'_j+\frac{\beta}{2}\|\hat{w}_{j}-w_j\|^2+\frac{2\bar{f}\sqrt{\log(2/\delta)}}{\sqrt{n}}\right)+\lambda^2\beta L^2H,
\end{aligned}
\end{equation} 
and we further have
\begin{equation}
\begin{split}
&\frac{1}{H}\sum_{j=1}^H\left( \phi(w_j)-\phi(\hat{w}_j)-\frac{\beta}{2}\|\hat{w}_{j}-w_j\|^2\right)\\
\leq& \Delta'_j+\frac{2\bar{f}\sqrt{\log(2/\delta)}}{\sqrt{n}}+\frac{\phi_{\frac{1}{2\beta}}(w_{1})-\min_w\phi(w)}{2\beta\lambda H}+\frac{\lambda L^2}{2},
\end{split}
\end{equation} 
Also, it holds that 
\begin{equation}
\begin{split}
  &\phi(w_j)-\phi(\hat{w}_j)-\frac{\beta}{2}\|\hat{w}_{j}-w_j\|^2\\
= & \phi(w_j)+\beta \|w_j-w_j\|^2-\phi(\hat{w}_j)-\beta\|\hat{w}_{j}-w_j\|^2+\frac{\beta}{2}\|\hat{w}_{j}-w_j\|^2\\
\geq & \frac{\beta}{2}\|\hat{w}_{j}-w_j\|^2=\frac{1}{4\beta}\|\nabla\phi_{\frac{1}{2\beta}}(w_j)\|^2,
\end{split}
\end{equation}
where the inequality holds by the definition of $\hat{w}_j$, and the last equality holds by the property of Moreau envelope such that $\frac{1}{2\beta}\nabla\phi_{\frac{1}{2\beta}}(w)=w-\hat{w}_j$ for any $w\in\mathcal{W}$. 
Therefore, we have 
\begin{equation}
\begin{split}
&\frac{1}{H}\sum_{j=1}^H \|\nabla\phi_{\frac{1}{2\beta}}(w_j)\|^2\\
\leq& 4\beta \frac{1}{H}\sum_{j=1}^H\Delta'_j+\frac{8\beta \bar{f}\sqrt{\log(2/\delta)}}{\sqrt{n}}+\frac{2(\phi_{\frac{1}{2\beta}}(w_{1})-\min_w\phi(w))}{\lambda H}+2\beta\lambda L^2,
\end{split}
\end{equation} 
By optimally choosing $\lambda=\sqrt{\frac{\phi_{\frac{1}{2\beta}}(w_{1})-\min_w\phi(w)}{\beta L^2 H}}$, we have
\begin{equation}
\begin{split}
&\mathbb{E}\left[\|\nabla\phi_{\frac{1}{2\beta}}(w_j)\|^2\right]\\
\leq& 4\beta \Delta'+\frac{8\beta \bar{f}\sqrt{\log(2/\delta)}}{\sqrt{n}}+4L\sqrt{\frac{(\phi_{\frac{1}{2\beta}}(w_{1})-\min_w\phi(w))\beta}{H}},
\end{split}
\end{equation}

\end{proof}

\subsubsection{Explanation of Convergence by Moreau envelop}\label{sec:Moreau_envelop}
The gradient bound of Moreau envelop indicates that the algorithm converges to an approximately stationary point, which is explained as below. 
The Moreau envelop $\nabla\phi_{\frac{1}{2\beta}}(w_j)$ satisfies 
$\nabla\phi_{\frac{1}{2\beta}}(w_j)=2\beta\cdot (w_j-\hat{w}_j)$. Since the proximal point is $\hat{w}_j=\arg\min_w \phi(w)+\beta \|w-w_j\|^2$, we have $\nabla\phi(\hat{w}_j)+2\beta (\hat{w}_j-w_j)=0$. Thus, it holds that $\nabla\phi_{\frac{1}{2\beta}}(w_j)=2\beta\cdot (w_j-\hat{w}_j)=\nabla\phi(\hat{w}_j)$ and $\|\hat{w}_j-w_j\|=\|\frac{1}{2\beta}\nabla\phi(\hat{w}_j)\|=\frac{1}{2\beta}\|\nabla\phi_{\frac{1}{2\beta}}(w_j)\|$.
Therefore, if the gradient of Moreau envelop is bounded for the decision variable $w$, i.e. $\|\nabla\phi_{\frac{1}{2\beta}}(w)\|\leq \Delta$, its proximal point $\hat{w}$ is an approximately stationary point for the optimal inner maximization function $\phi$ (bounded gradient of the inner maximization function $\|\nabla\phi(\hat{w})\|\leq \Delta$), and the distance between $w$ and $\hat{w}$ is close enough: 
$\|\hat{w}_j-w_j\|=\frac{1}{2\beta}\|\nabla\phi_{\frac{1}{2\beta}}(w_j)\|\leq \frac{\Delta}{2\beta}$. Thus, Algorithm \ref{DAG-DRL} approximately converges.

\section{Details of Experiments}\label{app:exp}
\subsection{Experiment Setups}

This section reports numerical experiments on a representative \ml task—time series forecasting based on the Electricity Maps \cite{electricitymaps_datasets_portal} datasets —to evaluate the effectiveness of the proposed algorithms.
 
    \subsubsection{Baselines}\label{sec:baselines}
    The baselines, which are compared with our algorithms in our experiments, are introduced as follows.
    
    \textbf{Standard ML} (\ml): This method trains the \ml to minimize the time series forecasting error without \dro. 

    \textbf{Diffusion-based ML} (\dml): 
    This is an \ml model fine-tuned with diffusion-generated augmented datasets. 
    Compared to \ouralg, \dml performs standard training based on the augmented datasets rather than a distributionally robust training. 
            
    \textbf{Wasserstein-based \dro} (\wdro): In this \dro framework, the ambiguity set is characterized by the Wasserstein metric. For our experiments, we employ the FWDRO algorithm proposed in \cite{staib2017distributionally}, which transforms the inner maximization of \wdro into an adversarial optimization with a mixed norm ball and then alternatively solves the adversarial examples and the \ml weights.
        
    \textbf{KL-divergence-based \dro} (\kldro): In this \dro framework, the ambiguity set is characterized by the \kl- divergence. We employ the standard \kldro solution derived in \cite{KL-DRO_hu2013kullback}, which is commonly adopted in practice.
    
    \subsubsection{Datasets}\label{sec:datasets}
    The experiments are conducted based on Electricity Maps \cite{electricitymaps_datasets_portal}, a widely utilized global platform that provides high-resolution spatio-temporal data on electricity system operations, including carbon intensity (gCO$_2$eq/kWh) and energy mix, and is actively employed for carbon-aware scheduling and carbon footprint estimation in real systems such as data centers \cite{electricitymaps_platform}.

    We utilize datasets from Electricity Maps that record hourly electricity carbon intensity over the period 2021$\sim$2024 across four representative regions: California, United States (\textbf{BANC\_21$\sim$24}), Texas, United States (\textbf{ERCO\_21$\sim$24}), Queensland, Australia (\textbf{QLD\_21$\sim$24}), and the United Kingdom (\textbf{GB\_21$\sim$24}). These datasets capture fine-grained temporal variations in carbon intensity (measured in gCO$_2$eq/kWh) arising from different energy mixes in diverse geographical and regulatory contexts, thereby providing a comprehensive benchmark for evaluating carbon-aware forecasting models. For model training, we construct a dataset by merging \textbf{BANC\_23} and \textbf{BANC\_24}, resulting in 438 sequence samples. Model evaluation is then performed on multiple independent test sets, each consisting of 312 sequence samples drawn from other years and regions to ensure heterogeneous and challenging testing scenarios. To quantify the degree of distributional shift between the training and test sets, we compute the Wasserstein distance, which provides a principled measure of discrepancy between probability distributions. The calculated distances are reported alongside the dataset names in Table~\ref{main_table_full}.

    \subsubsection{Training Setups} 
    
        The experimental setup is divided into the following parts:
        
        \textbf{Predictor}: The predictors in \ouralg and all the baselines share the same two-layer stacked LSTM architecture with 128 and 64 hidden neurons.

        \textbf{Diffusion Model}: The diffusion model in \ouralg is \ddpm \cite{NEURIPS2020_4c5bcfec} which has $T = 500$ steps in a forward or a backward process. 

        \textbf{Training}: For \ouralg, we adopt the PPO-based  reformulation in \eqref{eqn:ppo} for inner maximization. We train the reference \ddpm $\theta_0$ in \eqref{eqn:ppo} based on the original training dataset \textbf{BANC\_2324} ( \textbf{BANC\_23} \& \textbf{BANC\_24}) and use it to generate an initial dataset $\zinit$ to calculate $r_{\theta}$ in \eqref{eqn:ppo}. The sampling variance of \ddpm is chosen from a range $[0.1, 0.5]$. To improve training efficiency, only the last $T'=15$ backward steps of the \ddpm model are fine-tuned by \eqref{eqn:ppo}. We choose a slightly higher clipping parameter $\kappa=0.4$ in \eqref{eqn:ppo} to encourage the maximization while maintaining stability. We choose $\epsilon=0.015$ as \dagalg's adversarial budget, which gives the best average performance over all validation datasets. We choose $\eta = 0.01$ as the rate to update the Lagrangian parameter $\alpha$ in Algorithm \ref{DAG-DRL}. We use the Adam optimizer with a learning rate of $10^{-5}$ for both the diffusion training in the maximization and the predictor update in minimization. The diffusion model is trained for 10 inner epochs with a batch size of 64.  The predictor is trained for 15 epochs with a batch size of 64.

        For the baseline methods, we choose the same neural network architecture as \ouralg. We carefully tuned the hyperparameters of the baseline algorithms to achieve optimal average performance over all validation datasets. For \wdro, we consider the Wasserstein distance with respect to $l_2-$norm and set the adversarial budget as $\epsilon = 0.3$. For \kldro, we choose the adversarial budget $\epsilon = 4$.
        The predictors in all baseline methods are trained by Adam optimizer with a learning rate of $1 \times 10^{-5}$. 

        In terms of training, all baselines and \ouralg are initialized from the same \ml model pretrained for 100 epochs, using a batch size of 64 in both the pretraining and subsequent training phases. For our method, the outer minimization in \dro is performed for 15 iterations, and each outer iteration contains 10 inner maximization steps, during which the diffusion model is fine-tuned. After these 10 inner maximization steps, the current diffusion model generates a dataset $\ztheta$ of the same size as $\zinit$, and the pretrained \ml model is then further trained for two epochs on $\ztheta$. Consequently, our algorithm effectively fine-tunes the \ml model for 30 rounds in total, while utilizing 15 different augmented datasets $\ztheta$. To ensure a fully fair ablation setup, \dml is also fine-tuned for 30 rounds, but always with the same dataset $\zinit$. In contrast, since \wdro and \kldro are not part of the ablation study, they are trained for 100 epochs to achieve their best performance.

\subsection{Main results}

\begin{wraptable}{r}{0.6\textwidth}
\vspace{-8mm}
\centering
\caption{Test MSE on different datasets.}
\renewcommand{\arraystretch}{1.2}
\resizebox{0.6\textwidth}{!}{
\begin{tabular}{lccccc}
\toprule
\multicolumn{1}{c}{\multirow{2}{*}{\makecell{\textbf{Datasets} \\ (\textbf{Wasserstein Distance})}}}   & \multicolumn{5}{c}{\textbf{Algorithms}} \\
\cmidrule(lr){2-6}
 & \textbf{\ouralg} & \textbf{\kldro} & \textbf{\wdro} &  \textbf{\dml}  & \textbf{\ml}\\
\midrule

\multicolumn{1}{c}{\makecell{\textbf{BANC\_22}  (0.0240)}} & \textbf{0.0047} &	0.0086 &	0.0073 &	0.0078 &	0.0183   \\
\multicolumn{1}{c}{\makecell{\textbf{BANC\_21}  (0.1213)}} & \textbf{0.0054} &	0.0112 &	0.0121 &	0.0093 &	0.0238   \\
\midrule

\multicolumn{1}{c}{\makecell{\textbf{QLD\_24}  (0.2171)}} & \textbf{0.0450} &	0.0754 &	0.0823 &	0.0766 &	0.0887   \\
\multicolumn{1}{c}{\makecell{\textbf{QLD\_23}  (0.2033)}} & \textbf{0.0509} &	0.0831 &	0.0879 &	0.0834 & 	0.0946   \\
\multicolumn{1}{c}{\makecell{\textbf{QLD\_22}  (0.2782)}} & \textbf{0.0192} &	0.0379 &	0.0557 &	0.0352 &	0.0667   \\
\multicolumn{1}{c}{\makecell{\textbf{QLD\_21}  (0.3054)}} & \textbf{0.0186} &	0.0377 &	0.0574 &	0.0339 &	0.0696   \\
\midrule

\multicolumn{1}{c}{\makecell{\textbf{GB\_24}  (0.0419)}} & \textbf{0.0119} &	0.0178 &	0.0176 &	0.0172 &	0.0285   \\
\multicolumn{1}{c}{\makecell{\textbf{GB\_23}  (0.0666)}} & \textbf{0.0100} &	0.0178 &	0.0200 &	0.0164 &	0.0311   \\
\multicolumn{1}{c}{\makecell{\textbf{GB\_22}  (0.1255)}} & \textbf{0.0105} &	0.0197 &	0.0245 &	0.0172 &	0.0360    \\
\multicolumn{1}{c}{\makecell{\textbf{GB\_21}  (0.1359)}} & \textbf{0.0094} &	0.0181 &	0.0229 &	0.0158 &	0.0340    \\
\midrule

\multicolumn{1}{c}{\makecell{\textbf{ERCO\_24}  (0.1206)}} & \textbf{0.0158} &	0.0241 &	0.0266 &	0.0224 &	0.0379   \\
\multicolumn{1}{c}{\makecell{\textbf{ERCO\_23}  (0.1207)}} & \textbf{0.0106} &	0.0179 &	0.0196 &	0.0162 &	0.0319    \\
\multicolumn{1}{c}{\makecell{\textbf{ERCO\_22}  (0.1581)}} & \textbf{0.0076} &	0.0146 &	0.0187 &	0.0123 &	0.0312    \\
\multicolumn{1}{c}{\makecell{\textbf{ERCO\_21}  (0.1417)}} & \textbf{0.0093} &	0.0189 &	0.0263 &	0.0160 &	0.0382   \\
\midrule

\multicolumn{1}{c}{\textbf{Average}} & \textbf{0.0163} &	0.0288 &	0.0342 &	0.0271 &	0.0450    \\
\midrule
\multicolumn{1}{c}{\textbf{Maximum}} & \textbf{0.0509} &	0.0831 &	0.0879 &	0.0834 &	0.0946   \\

\bottomrule
\end{tabular}
}
\label{main_table_full}
\end{wraptable}

     We tune the hyperparameters of all methods based on validation datasets and present their best performance  in Table~\ref{main_table_full} where all the datasets are \ood testing datasets with different discrepancies.

    We can find that all \dro methods improve the testing performance for \ood datasets compared to \ml by optimizing the worst-case expected loss, while \dml improves upon \ml by leveraging diffusion-generated augmented datasets.
    Using the performance of \ml as the reference, \ouralg achieves the largest performance gain of 63.7\%, followed by \dml with a 39.7\% improvement. In contrast, \kldro and \wdro yield improvements of 36.1\% and 24.0\%, respectively.

    Notably, \dml surpasses \kldro by 3.6\% only after augmentation training with diffusion-generated datasets, highlighting the significant positive impact of the diffusion model. Moreover, \dml differs from \ouralg only in that it disables the \dro component, while their training procedures remain identical (see Appendix ~\ref{app:exp} for details). Thus, \ouralg, \dml, and \ml together form two ablation studies: the performance gap between \ouralg and \dml confirms that \dro contributes a 39.7\% performance gain to \ouralg, while the gap between \dml and \ml verifies that the diffusion model also provides a 39.7\% performance gain.

    Although both \dml and the other two \dro deliver noticeable performance gains, \ouralg still outperforms them by an average margin of 30.4\%. The underlying reason may lie in the fact that \ouralg leverages the distribution learning capability of the diffusion model when constructing the ambiguity set, enabling the generation of adversarial distributions that are both strong and realistic, thereby achieving superior \ood generalization. In contrast, the \kl-divergence used in \kldro requires all distributions to be absolutely continuous with respect to the training distribution $P_0$, which constrains the support space of the ambiguity set and limits the search for strong adversaries. Among the baselines, \wdro performs the worst, likely due to the fact that solving Wasserstein-based \dro usually involves optimal transport problems or their relaxations, which are computationally more demanding and often rely on approximate methods such as dual reformulation or adversarial training, resulting in suboptimal solutions. Moreover, since \wdro allows adversarial distributions with supports different from the training distribution, it is theoretically more flexible and closer to real distribution shifts, but this flexibility can produce overly extreme adversaries and thus lead to overly conservative model training.

\subsection{More Out-Of-Distribution Tests}
\subsubsection{Effect of Noisy Types}

\begin{wraptable}{r}{0.5\textwidth}
\vspace{-8mm}
\centering
\caption{Gaussian-Corrupted Test.}
\renewcommand{\arraystretch}{1.2}
\resizebox{0.5\textwidth}{!}{%
\begin{tabular}{lccccc}
\toprule
\multicolumn{1}{c}{\multirow{2}{*}{\textbf{Dataset}}}   & \multicolumn{5}{c}{\textbf{Algorithms}} \\
\cmidrule(lr){2-6}
 & \textbf{\ouralg} & \textbf{\kldro} & \textbf{\wdro} &  \textbf{\dml}  & \textbf{\ml}\\
\midrule

\multicolumn{1}{c}{\makecell{\textbf{BANC\_22}}} & \textbf{0.0170} &	0.0197 &	0.0183 &	0.0189 &	0.0291    \\
\multicolumn{1}{c}{\makecell{\textbf{BANC\_21}}} & \textbf{0.0177} &	0.0214 &	0.0216 &	0.0199 &	0.0337    \\
\midrule

\multicolumn{1}{c}{\makecell{\textbf{QLD\_24}}} & \textbf{0.0560} &	0.0839 &	0.0911 &	0.0847 &	0.0990    \\
\multicolumn{1}{c}{\makecell{\textbf{QLD\_23}}} & \textbf{0.0615} &	0.0925 &	0.0960 &	0.0934 &	0.1034    \\
\multicolumn{1}{c}{\makecell{\textbf{QLD\_22}}} & \textbf{0.0307} &	0.0476 &	0.0639 &	0.0445 &	0.0766    \\
\multicolumn{1}{c}{\makecell{\textbf{QLD\_21}}} & \textbf{0.0297} &	0.0475 &	0.0649 &	0.0431 &	0.0772    \\
\midrule

\multicolumn{1}{c}{\makecell{\textbf{GB\_24}}} & \textbf{0.0238} &	0.0280 &	0.0276 &	0.0279 &	0.0381    \\
\multicolumn{1}{c}{\makecell{\textbf{GB\_23}}} & \textbf{0.0221} &	0.0260 &	0.0305 &	0.0269 &	0.0414    \\
\multicolumn{1}{c}{\makecell{\textbf{GB\_22}}} & \textbf{0.0227} &	0.0302 &	0.0343 &	0.0279 &	0.0458     \\
\multicolumn{1}{c}{\makecell{\textbf{GB\_21}}} & \textbf{0.0211} &	0.0279 &	0.0334 &	0.0256 &	0.0438     \\
\midrule

\multicolumn{1}{c}{\makecell{\textbf{ERCO\_24}}} & \textbf{0.0281} &	0.0350 &	0.0369 &	0.0340 &	0.0471    \\
\multicolumn{1}{c}{\makecell{\textbf{ERCO\_23}}} & \textbf{0.0228} &	0.0281 &	0.0294 &	0.0267 &	0.0418     \\
\multicolumn{1}{c}{\makecell{\textbf{ERCO\_22}}} & \textbf{0.0206} &	0.0256 &	0.0293 &	0.0230 &	0.0416     \\
\multicolumn{1}{c}{\makecell{\textbf{ERCO\_21}}} & \textbf{0.0217} &	0.0295 &	0.0357 &	0.0264 &	0.0475    \\
\midrule

\multicolumn{1}{c}{\textbf{Average}} & \textbf{0.0282} &	0.0388 &	0.0438 &	0.0373 &	0.0547     \\
\midrule
\multicolumn{1}{c}{\textbf{Maximum}} & \textbf{0.0615} &	0.0925 &	0.0960 &	0.0934 &	0.1034    \\
\bottomrule

\end{tabular}
}
\label{gaussian_table}
\end{wraptable}

In this test, we add a certain amount of different types of noise into each test set to compare the noise robustness of different algorithms. The noise types include Gaussian Noise, Perlin Noise, and Cutout Noise.

\textbf{Gaussian Noise}: In the Gaussian Noise test, we add Gaussian Noise with $\sigma=0.1$ to each test set. As shown in Table~\ref{gaussian_table}, all algorithms exhibit performance degradation compared to the noise-free setting. Nevertheless, \ouralg still significantly outperforms the others: taking \ml as the reference, our method achieves a 48.3\% improvement, followed by \dml and \kldro with gains of 31.7\% and 29.15\%, respectively. The weakest performer is \wdro, which surpasses \ml by only 20\%. In fact, \wdro is relatively adept at handling Gaussian Noise compared to other noise types, and thus maintains a noticeable advantage even under $\sigma=0.1$ Gaussian Noise, a trend further confirmed in subsequent experiments. On the other hand, both \dml and \ouralg are trained with diffusion-generated augmented datasets, which inherently possess noise characteristics due to the Gaussian-based diffusion process. As a result, their performance remains robust under Gaussian Noise perturbations.

\textbf{Perlin Noise}: Perlin Noise is a smooth pseudo-random gradient noise commonly used to simulate natural textures such as clouds, terrains, and wood grains. By combining multiple Perlin Noise components with different frequencies and amplitudes (known as octaves), more complex fractal noise can be produced. In this experiment, we superimpose 8 layers of Perlin Noise, with the noise amplitude normalized to the range [-1, 1]. As shown in Table~\ref{perlin_table}, taking \ml as the reference, our method outperforms \ml by 76.4\%, while \dml achieves a 58.0\% improvement, and both perform better than in the Gaussian Noise test. Although \kldro also shows a larger gain relative to \ml, its MSE remains roughly the same as in the Gaussian Noise test; the apparent improvement is primarily due to the substantial performance drop of \ml in this experiment. 

\begin{wraptable}{r}{0.5\textwidth}
\vspace{-5mm}
\centering
\caption{Perlin-Corrupted Test.}
\renewcommand{\arraystretch}{1.2}
\resizebox{0.5\textwidth}{!}{
\begin{tabular}{lccccc}
\toprule
\multicolumn{1}{c}{\multirow{2}{*}{\textbf{Dataset}}}   & \multicolumn{5}{c}{\textbf{Algorithms}} \\
\cmidrule(lr){2-6}
 & \textbf{\ouralg} & \textbf{\kldro} & \textbf{\wdro} &  \textbf{\dml}  & \textbf{\ml}\\
\midrule

\multicolumn{1}{c}{\makecell{\textbf{BANC\_22}}} & \textbf{0.0117} &	0.0296 &	0.0620 &	0.0227 &	0.0663     \\
\multicolumn{1}{c}{\makecell{\textbf{BANC\_21}}} & \textbf{0.0110} &	0.0288 &	0.0591 &	0.0211 &	0.0662     \\
\midrule

\multicolumn{1}{c}{\makecell{\textbf{QLD\_24}}} & \textbf{0.0355} &	0.0636 &	0.0862 &	0.0588 &	0.0928     \\
\multicolumn{1}{c}{\makecell{\textbf{QLD\_23}}} & \textbf{0.0406} &	0.0685 &	0.0787 &	0.0669 &	0.0860     \\
\multicolumn{1}{c}{\makecell{\textbf{QLD\_22}}} & \textbf{0.0186} &	0.0476 &	0.0852 &	0.0341 &	0.0904     \\
\multicolumn{1}{c}{\makecell{\textbf{QLD\_21}}} & \textbf{0.0192} &	0.0475 &	0.0865 &	0.0355 &	0.0940     \\
\midrule

\multicolumn{1}{c}{\makecell{\textbf{GB\_24}}} & \textbf{0.0147} &	0.0272 &	0.0499 &	0.0252 &	0.0563     \\
\multicolumn{1}{c}{\makecell{\textbf{GB\_23}}} & \textbf{0.0133} &	0.0295 &	0.0560 &	0.0239 &	0.0628     \\
\multicolumn{1}{c}{\makecell{\textbf{GB\_22}}} & \textbf{0.0132} &	0.0311 &	0.0611 &	0.0247 &	0.0666      \\
\multicolumn{1}{c}{\makecell{\textbf{GB\_21}}} & \textbf{0.0114} &	0.0311 &	0.0586 &	0.0224 &	0.0677      \\
\midrule

\multicolumn{1}{c}{\makecell{\textbf{ERCO\_24}}} & \textbf{0.0154} &	0.0267 &	0.0436 &	0.0238 &	0.0525     \\
\multicolumn{1}{c}{\makecell{\textbf{ERCO\_23}}} & \textbf{0.0142} &	0.0382 &	0.0767 &	0.0262 &	0.0829      \\
\multicolumn{1}{c}{\makecell{\textbf{ERCO\_22}}} & \textbf{0.0109} &	0.0324 &	0.0655 &	0.0223 &	0.0732      \\
\multicolumn{1}{c}{\makecell{\textbf{ERCO\_21}}} & \textbf{0.0104} &	0.0283 &	0.0506 &	0.0194 &	0.0605     \\
\midrule

\multicolumn{1}{c}{\textbf{Average}} & \textbf{0.0171} &	0.0379 &	0.0657 &	0.0305 &	0.0727      \\
\midrule
\multicolumn{1}{c}{\textbf{Maximum}} & \textbf{0.0406} &	0.0685 &	0.0865 &	0.0669 &	0.0940     \\
\bottomrule
\end{tabular}
}
\label{perlin_table}
\end{wraptable}

 In contrast, \wdro outperforms \ml by only 9.6\%, a significant decline compared to its performance under Gaussian Noise, with the MSE reduced by as much as 50\%.  This indicates that \wdro is not well-suited for handling Perlin Noise, likely because the Wasserstein distance measures global distributional transport cost and is more effective in capturing smooth, small perturbations (e.g., Gaussian Noise), but fails to adequately model the long-range correlated patterns of Perlin Noise within the Wasserstein ball.

\textbf{Cutout Noise}: Cutout Noise is a commonly used perturbation method that randomly selects a region of the input data and sets its values to a constant, thereby simulating partial information loss.  In our experiment, we randomly mask 30\% of the sequence and set the masked values to a constant of 1. As shown in Table~\ref{cutout_table}, the performance of all algorithms is very close to their performance under Perlin Noise. 

\begin{wraptable}{r}{0.5\textwidth}
\vspace{-10mm}
\centering
\caption{Cutout-Corrupted Test.}
\renewcommand{\arraystretch}{1.2}
\resizebox{0.5\textwidth}{!}{
\begin{tabular}{lccccc}
\toprule
\multicolumn{1}{c}{\multirow{2}{*}{\textbf{Dataset}}}   & \multicolumn{5}{c}{\textbf{Algorithms}} \\
\cmidrule(lr){2-6}
 & \textbf{\ouralg} & \textbf{\kldro} & \textbf{\wdro} &  \textbf{\dml}  & \textbf{\ml}\\
\midrule

\multicolumn{1}{c}{\makecell{\textbf{BANC\_22}}} & \textbf{0.0063} &	0.0181 &	0.0297 &	0.0125 &	0.0385      \\
\multicolumn{1}{c}{\makecell{\textbf{BANC\_21}}} & \textbf{0.0095} &	0.0302 &	0.0547 &	0.0192 &	0.0637      \\
\midrule

\multicolumn{1}{c}{\makecell{\textbf{QLD\_24}}} & \textbf{0.0404} &	0.0798 &	0.1096 &	0.0697 &	0.1148      \\
\multicolumn{1}{c}{\makecell{\textbf{QLD\_23}}} & \textbf{0.0426} &	0.0802 &	0.1028 &	0.0743 &	0.1092      \\
\multicolumn{1}{c}{\makecell{\textbf{QLD\_22}}} & \textbf{0.0196} &	0.0507 &	0.0881 &	0.0384 &	0.0971      \\
\multicolumn{1}{c}{\makecell{\textbf{QLD\_21}}} & \textbf{0.0208} &	0.0534 &	0.0916 &	0.0400 &	0.1005      \\
\midrule

\multicolumn{1}{c}{\makecell{\textbf{GB\_24}}} & \textbf{0.0145} &	0.0351 &	0.0598 &	0.0252 &	0.0688      \\
\multicolumn{1}{c}{\makecell{\textbf{GB\_23}}} & \textbf{0.0122} &	0.0343 &	0.0603 &	0.0230 &	0.0681      \\
\multicolumn{1}{c}{\makecell{\textbf{GB\_22}}} & \textbf{0.0129} &	0.0361 &	0.0639 &	0.0245 &	0.0720       \\
\multicolumn{1}{c}{\makecell{\textbf{GB\_21}}} & \textbf{0.0127} &	0.0356 &	0.0651 &	0.0245 &	0.0740       \\
\midrule

\multicolumn{1}{c}{\makecell{\textbf{ERCO\_24}}} & \textbf{0.0157} &	0.0375 &	0.0643 &	0.0275 &	0.0732      \\
\multicolumn{1}{c}{\makecell{\textbf{ERCO\_23}}} & \textbf{0.0134} &	0.0349 &	0.0599 &	0.0241 &	0.0678       \\
\multicolumn{1}{c}{\makecell{\textbf{ERCO\_22}}} & \textbf{0.0134} &	0.0349 &	0.0599 &	0.0241 &	0.0678       \\
\multicolumn{1}{c}{\makecell{\textbf{ERCO\_21}}} & \textbf{0.0116} &	0.0354 &	0.0671 &	0.0225 &	0.0758      \\
\midrule

\multicolumn{1}{c}{\textbf{Average}} & \textbf{0.0174} &	0.0425 &	0.0699 &	0.0319 &	0.0782       \\
\midrule
\multicolumn{1}{c}{\textbf{Maximum}} & \textbf{0.0426} &	0.0802 &	0.1096 &	0.0743 &	0.1148      \\
\bottomrule

\end{tabular}
}
\label{cutout_table}
\end{wraptable}

 We attribute this to the fact that, although Perlin and Cutout Noises differ in form, both represent structured local perturbations that disrupt the continuity of the input patterns, thereby posing similar challenges to all 
algorithms and resulting in comparable performance under these two types of noise at a given perturbation level. This is further confirmed in the subsequent gradient-perturbation tests.

\subsubsection{Effect of Noisy Levels}

In this experiment, we progressively increased the intensity of three types of noise. For Gaussian Noise, the perturbation range is set to $\sigma \in [0.05, 0.2]$; for Perlin Noise, the amplitude is controlled within the range $[0.05, 1]$; and for Cutout Noise, the Cutout Mask Ratio is adjusted between $[10\%, 40\%]$. As shown in Figures~\ref{Grad_Gaussian},~\ref{Grad_Perlin},~\ref{Grad_Cutout}, \ouralg consistently outperforms all baseline methods across different noise types and intensity levels. With increasing noise strength, we observe that variations in Gaussian Noise have a more pronounced impact on all methods compared to Perlin and Cutout Noise. In contrast, the impact of stronger Perlin and Cutout Noise on \ouralg remains limited, and even at higher noise levels, \ouralg maintains stable and superior performance, highlighting its strong robustness. For \wdro, however, the performance degradation under Perlin and Cutout Noise is much greater, with trends almost identical to \ml, indicating that \wdro is not effective in handling Perlin and Cutout Noise but is relatively better at coping with Gaussian Noise. Interestingly, when the Cutout Mask Ratio is 30\%, the performance of all algorithms is nearly identical to their performance under Perlin Noise with amplitude 1, suggesting that at specific noise levels, Perlin and Cutout—though different in form—both represent structured local perturbations that disrupt input continuity to a similar degree, thereby producing comparable impacts on the algorithms.

    \begin{figure*}[htbp]
        \centering
        \resizebox{1\textwidth}{!}{%
            \parbox{\textwidth}{%
                \begin{minipage}[b]{0.49\textwidth}
                    \centering
                    \includegraphics[width=\textwidth]{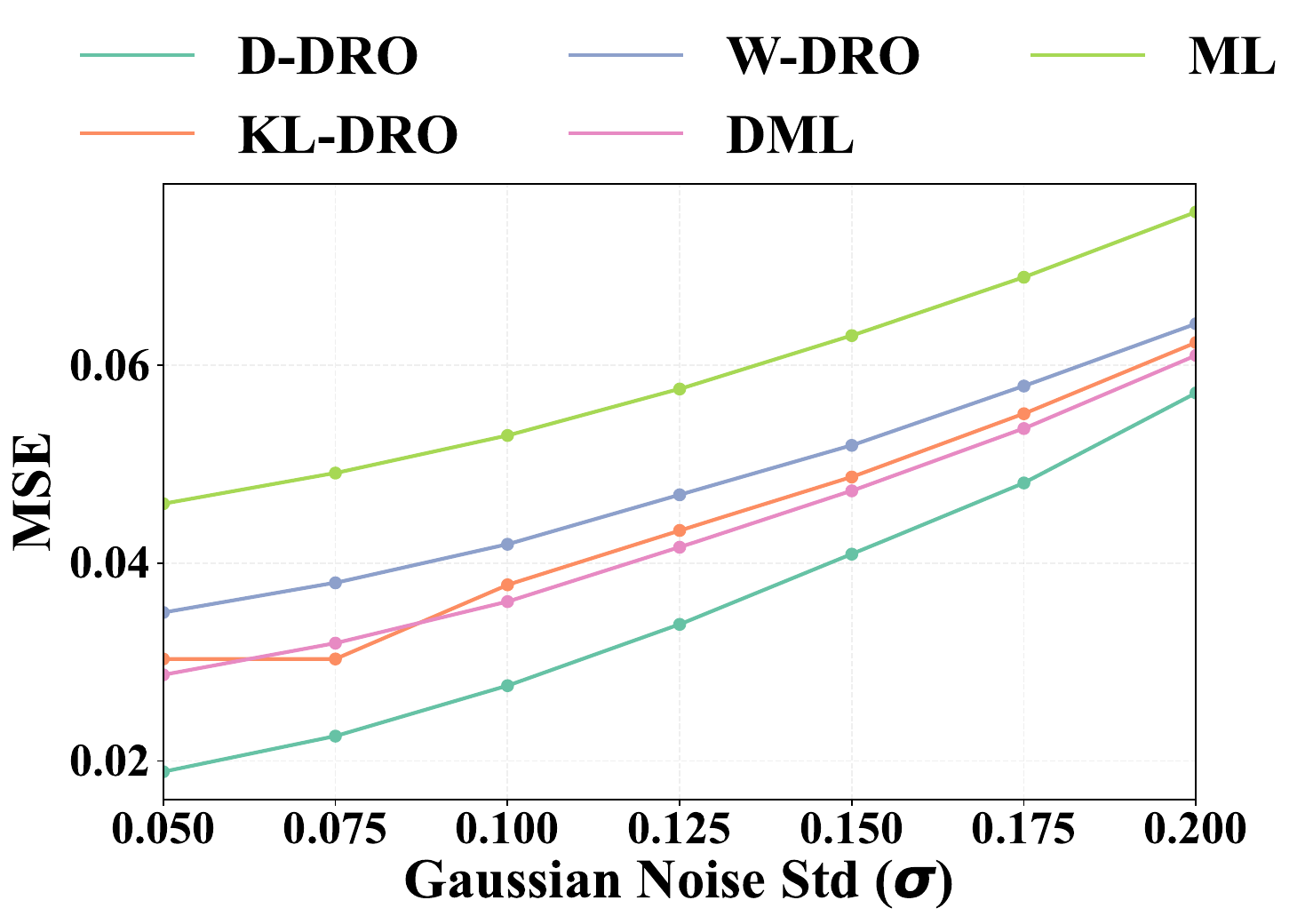}
                    \caption{Gaussian perturbation strength}
                    \label{Grad_Gaussian}
                \end{minipage}
                \hfill
                \begin{minipage}[b]{0.49\textwidth}
                    \centering
                    \includegraphics[width=\textwidth]{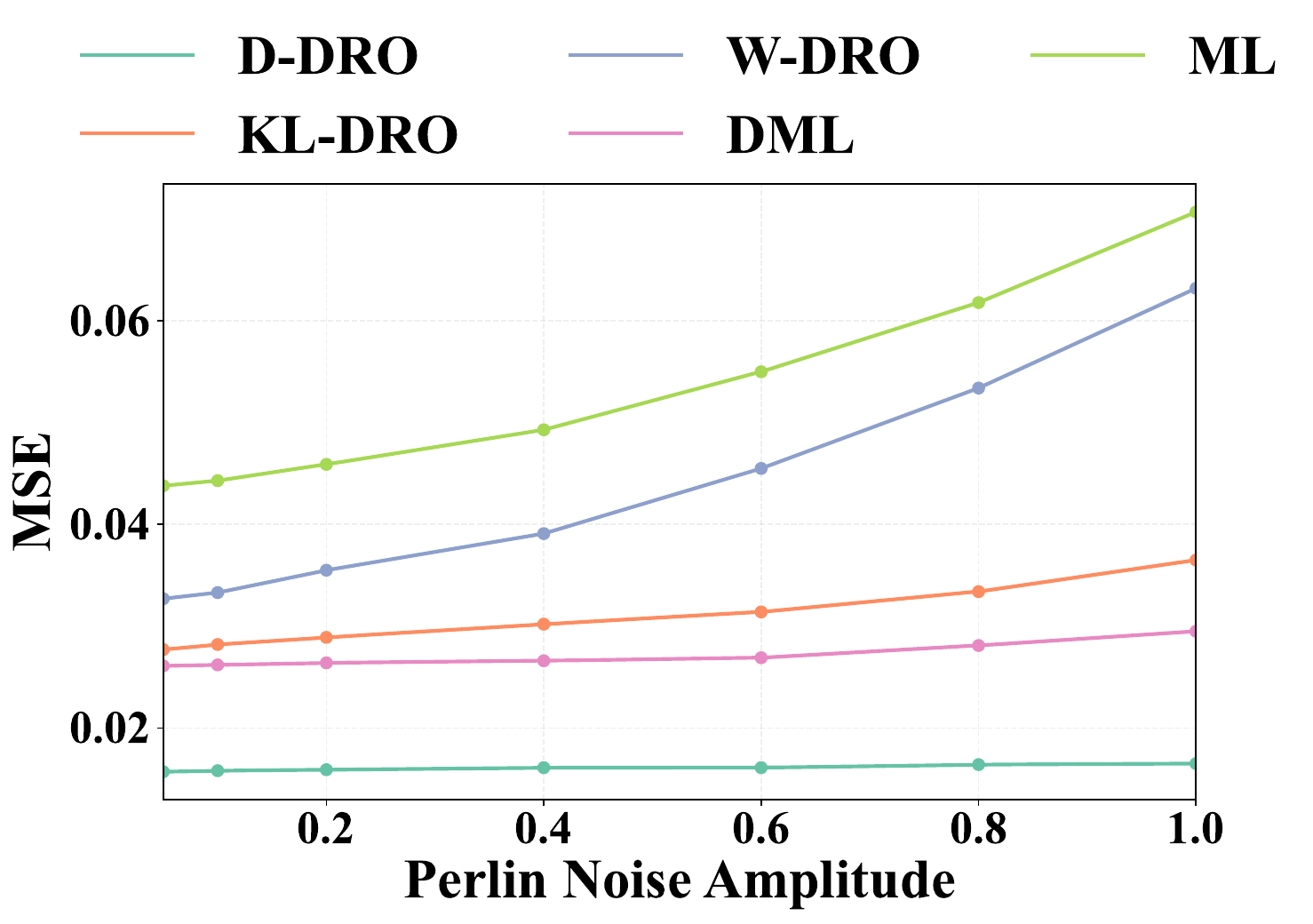}
                    \caption{Perlin perturbation strength}
                    \label{Grad_Perlin}
                \end{minipage}
                \vfill
                \begin{minipage}[b]{0.49\textwidth}
                    \centering
                    \includegraphics[width=\textwidth]{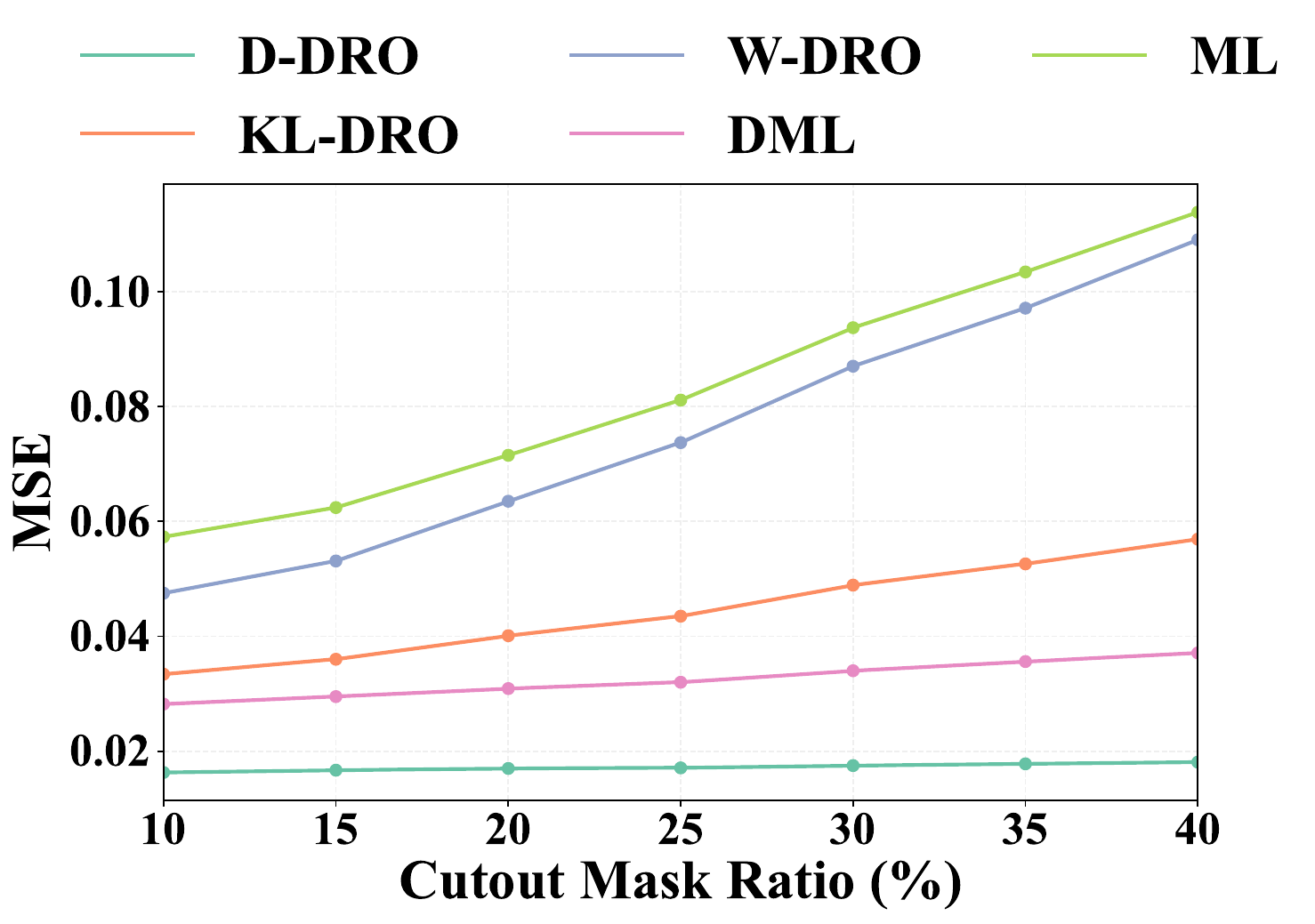}
                    \caption{Cutout perturbation strength}
                    \label{Grad_Cutout}
                \end{minipage}
                \hfill
                \begin{minipage}[b]{0.49\textwidth}
                    \centering
                    \includegraphics[width=\textwidth]{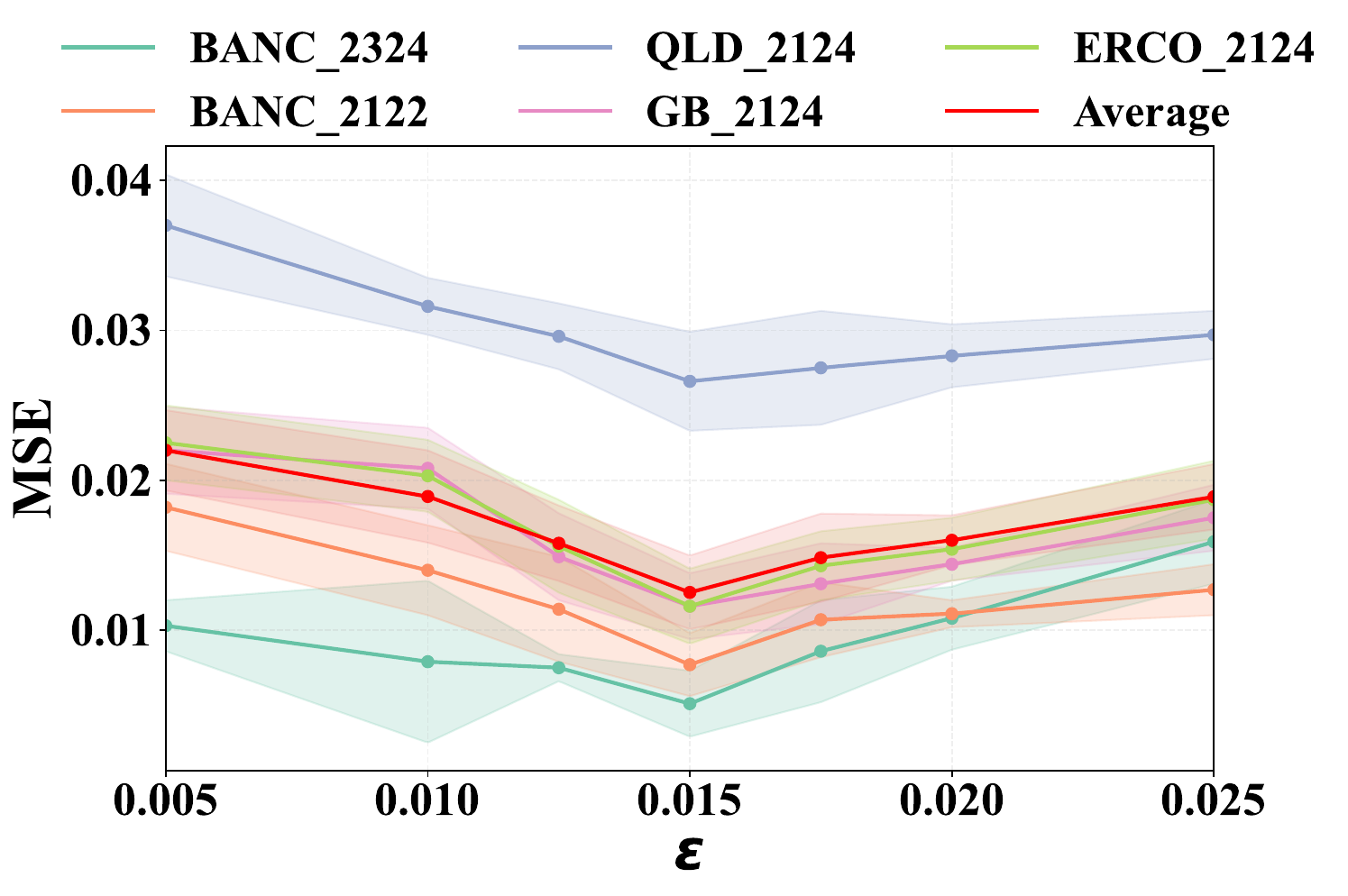}
                    \caption{Effect of budget $\epsilon$ in \ouralg }
                    \label{eps_select}
                \end{minipage}
            }
        }
    \end{figure*}

\subsection{Effects of \dro Budget}

    Finally, we examine the impact of the budget parameter $\epsilon$ in \eqref{eqn:objective} on the performance of \ouralg. As illustrated in Fig.~\ref{eps_select}, the loss–$\epsilon$ curves across all datasets display a concave trend, with the best average performance achieved around $\epsilon = 0.015$. When $\epsilon$ is smaller than this threshold, the diffusion-modelled distributions are overly restricted to the training data, thereby hindering the ability of \ouralg to generalize to \ood datasets. In contrast, when $\epsilon$ becomes excessively large, the enlarged ambiguity set causes \ouralg to conservatively optimize against irrelevant distributions, which degrades its performance on real \ood datasets.
    Hence, selecting an appropriate value of $\epsilon$ is essential for constructing effective adversarial distributions, ensuring a proper balance between average-case and worst-case performance.


\end{document}